\documentclass{article}
\usepackage{algpseudocode,algorithm,algorithmicx}
\newcommand*\Let[2]{\State #1 $\gets$ #2}
\algrenewcommand\algorithmicrequire{\textbf{Require:}}
\algrenewcommand\algorithmicensure{\textbf{Returns:}}
\usepackage{microtype}
\usepackage{graphicx}
\usepackage{subfigure}
\usepackage{booktabs} 

\usepackage{hyperref}

\usepackage[accepted]{icml2021}

\usepackage{amsmath}
\usepackage{amsthm}
\usepackage{thmtools,thm-restate}
\usepackage{wrapfig}
\usepackage{amsfonts}
\usepackage{amssymb}
\usepackage{accents}
\usepackage{cleveref}
\usepackage{comment}
\usepackage{svg}
\usepackage{placeins}

\usepackage[english]{babel}

\newcommand{\norm}[1]{\left\lVert#1\right\rVert}
\newcommand{\hatcal}[1]{\hat{\mathcal{#1}}}

\newtheorem{definition}{Definition}

\DeclareMathOperator*{\argmax}{arg\,max}

\DeclareMathOperator{\R}{\mathbb{R}}
\providecommand{\e}[1]{\ensuremath{\times 10^{#1}}}

\icmltitlerunning{Provable Lipschitz Certification for Generative Models}

\begin{document}

\twocolumn[
\icmltitle{Provable Lipschitz Certification for Generative Models}



\icmlsetsymbol{equal}{*}

\begin{icmlauthorlist}
\icmlauthor{Matt Jordan}{ut}
\icmlauthor{Alexandros G. Dimakis}{ut}
\end{icmlauthorlist}

\icmlaffiliation{ut}{University of Texas at Austin}

\icmlcorrespondingauthor{Matt Jordan}{mjordan@cs.utexas.edu}

\icmlkeywords{Machine Learning, ICML}

\vskip 0.3in
]



\printAffiliationsAndNotice{}  
\begin{abstract}
We present a scalable technique for upper bounding the Lipschitz constant of generative models. We relate this quantity to the maximal norm over the set of attainable vector-Jacobian products of a given generative model. We approximate this set by layerwise convex approximations using zonotopes. Our approach generalizes and improves upon prior work using zonotope transformers and we extend to Lipschitz estimation of neural networks with large output dimension. This provides efficient and tight bounds on small networks and can scale to generative models on VAE and DCGAN architectures. 
\end{abstract}

\section{Introduction}
We study the problem of bounding the Lipschitz constant of generative models. The central technical difficulty is that these are vector-valued functions with high-dimensional outputs, so the techniques for Lipschitz estimation of scalar-valued neural networks do not directly translate. Our approach is to frame the Lipschitz constant estimation problem as an optimization over the range of attainable vector-Jacobian products. We overapproximate this set via layerwise convex approximations, and then solve a relaxed version of the optimization problem. 

The primary challenge is computing the feasible set of vector-Jacobian products over a range of inputs and vectors. We generate an overapproximation of this set by representing it as a zonotope. While prior work has performed reachability analysis of neural networks using zonotopes, these approaches only consider the forward pass. Our technique generalizes this prior work and is able to yield tighter bounds. Additionally, we are able to apply our approach to backpropagation, where we pass zonotopes backwards in place of vectors.

We present a general algorithm for mapping zonotopes through a variety of nonlinear operators and relate this to a 2-dimensional geometric problem that we solve optimally. The final norm-maximization that arises turns out to be 
equivalent to a mathematical question called the Grothendieck problem. We demonstrate that the linear-programming relaxation of this problem can be solved in linear time with our machinery. To fairly compare our work against existing Lipschitz estimation techniques, we first compare the estimated Lipschitz value and runtime of our algorithm on networks of increasing size trained on a toy dataset. We observe that our approach favorably trades-off accuracy for efficency and yields tighter bounds 
compared to previous techniques. 
Further, it can significantly improve the runtime of exact Lipschitz computations.  We scale our technique to generative models on MNIST and CIFAR-10, using well-known architectures like DCGAN and VAEs with both fully-connected and convolutional layers \cite{Kingma2013-en, Radford2015-hi}. Our approach yields much tighter bounds compared to any other technique that can handle vector-valued networks and can scale to such architecture sizes.

\section{Related Work}
\paragraph{Robustness Certification:} 
Lipschitz estimation is closely related to robustness certification. Here, the goal is to provide a certificate of robustness against adversarial attacks for a specified network and input region. The techniques of interest in this domain are Lipschitz approximation and reachability analysis.  It has been noted multiple times that an upper bound to the Lipschitz constant can provide a guarantee of robustness against adversarial examples \cite{Hein2017-ky, Weng2018-lf, Weng2018-gr}. Reachability analysis is frequently couched in the language of abstract interpretations where the goal is to develop sound transformations
to map sets through the forward pass of a given classifier. The classes of sets considered include hyperboxes and zonotopes, as well as polytopes, imageStars, and linear bounds \cite{Singh2019-mf, SinghGagandeep2019-ki, Mirman2018-nx, Zico_Kolter2017-va, Weng2018-gr, Xu2020-du, Zhang2018-fl, Singh2018-tg, Tran2020-hl}. 

\paragraph{Lipschitz Approximation:} A number of recent works provide either heuristic estimates or provable upper bounds of the Lipschitz constant of neural networks. For ReLU networks, this problem is known to be NP-hard and has strong inapproximability guarantees \cite{Virmaux2018-ti, Jordan2020-tq}. However for small networks it has been shown that this quantity may be estimated to a reasonable degree of accuracy. The majority of these works are either unable to handle vector-valued neural networks or are unable to scale to larger networks. Provable upper bounds of the Lipschitz constant may be attained using interval analysis \cite{Weng2018-gr}, semidefinite programming \cite{Fazlyab2019-im}, linear or mixed-integer programming \cite{Jordan2019-sw}, or polynomial optimization \cite{noauthor_undated-un, Chen2020-ul}. Heuristic estimates of the Lipschitz constant may be attained by either randomness and extreme-value theory \cite{Weng2018-lf}, or by a greedy algorithm over the activation patterns \cite{Virmaux2018-ti}. Our work can be viewed as a spiritual successor to the interval analysis approach of Lipschitz estimation and the zonotope abstract interpretation approaches for robustness certification.

\section{Problem Statement and Relaxation Strategy} 
We formally define the problem of estimating the Lipschitz constants of vector-valued functions and provide a broad overview of the strategy for our convex relaxation. We start with the notation we will employ. 
\paragraph{Notation:}
We will denote vectors using lowercase letters, matrices using capital letters and sets using calligraphic letters. We denote the unit-norm ball with respect to the $\alpha$-norm as $B_\alpha$. We refer to the dual norm of the $\alpha$-norm as the norm $\norm{\cdot}_{\alpha^*}$, defined as $\norm{v}_{\alpha^*}:= \sup_{u\in B_\alpha} |u^Tv|$. For a function $f(\cdot): \R^k \to \R^n$, we denote the Jacobian with respect to its argument as $\nabla_x f(x)$, which is a matrix in $\R^{n\times k}$. We will frequently abuse notation and for sets $\mathcal{X}$ and functions $f$, we write $f(\mathcal{X})$ to denote the set $\{f(x)\mid x\in \mathcal{X}\}$. We use $\odot$ to denote the Hadamard product of two vectors, and $\oplus$ to denote the Minkowksi sum of two sets.

\paragraph{Lipschitz constants of vector-valued functions:}
Given a neural network $f: \R^{k}\rightarrow \R^n$, an input domain $\mathcal{X}\subseteq \R^k$ and norms $\norm{\cdot}_\alpha$, $\norm{\cdot}_\beta$ over $\R^k, \R^n$ respectively, the local Lipschitz constant is formally defined as:
\begin{equation}
    L^{(\alpha, \beta)}(f, \mathcal{X}) := \sup\limits_{x, y\in\mathcal{X}} \frac{\norm{f(x)-f(y)}_\beta}{\norm{x-y}_\alpha}.
\end{equation}
When $f$ is continuously differentiable and $\mathcal{X}$ is an open set, this quantity may be written as an optimization over matrix norms of the Jacobian. Letting $\nabla_x f(x)$ be the Jacobian of $f$ evaluated at $x$, we have that 
\begin{equation}\label{eq:jacobian-rewrite}
    L^{(\alpha, \beta)}(f, \mathcal{X}) = \sup_{x\in\mathcal{X}} \norm{\nabla_x f(x)}_{\alpha\to \beta}.
\end{equation}
The matrix norm $\norm{M}_{\alpha \to \beta}$ is defined as 
\begin{equation}\label{eq:matrix-norm}
\norm{M}_{\alpha\to\beta} := \sup_{\norm{v}_\alpha \leq 1} \norm{Mv}_\beta = \sup_{\norm{u}_{\beta^*} \leq 1} \norm{M^Tu}_{\alpha^*}.
\end{equation}
where the second equality follows from the definition of the dual norm. Combining Equations \ref{eq:jacobian-rewrite} and \ref{eq:matrix-norm}, we can formulate the problem of computing the Lipschitz constant as an optimization over vector-Jacobian products: 
\begin{equation}\label{eq:problem-reformulate}
    L^{(\alpha, \beta)}(f, \mathcal{X}) = \sup\limits_{x\in\mathcal{X}} \sup\limits_{u \in B_{\beta^*}} \norm{\nabla_x f(x)^Tu}_{\alpha*}.
\end{equation}
When $f$ is nonsmooth, as in the case of neural networks with ReLU nonlinearities, the optimization over the Jacobian in Equation \ref{eq:problem-reformulate} is replaced with an optimization over Clarke generalized subgradients \cite{Jordan2020-tq}. For our purposes, as we seek only to upper bound this quantity, this distinction is of minimal importance.

The strategy we employ to upper bound Equation \ref{eq:problem-reformulate} will rely on two clear steps. The first step is to generate a sound approximation of the set of vector-Jacobian products of $f$, and the second step is to bound the maximal $\norm{\cdot}_{\alpha^*}$ norm of this set. Concisely, we first develop a set $\mathcal{Y}$ satisfying the containment
\begin{equation}\label{eq:main-containment}
    \{\nabla_xf(x)^Tu \mid x\in\mathcal{X},\quad u\in B_{\beta^*}\} \subseteq \mathcal{Y},
\end{equation}
and then we upper bound the maximal $\alpha^*$ norm of $\mathcal{Y}$.

We will focus in particular on the $L^{(\infty, 1)}(f, \mathcal{X})$ Lipschitz constant. The choice of $\norm{\cdot}_\alpha$ to be the $\ell_\infty$ norm is standard in the robustness certification literature. The choice of $\norm{\cdot}_\beta$ to be the $\ell_1$ norm will yield an upper bound for $L^{(\alpha, p)}(f,\mathcal{X})$ for $p \geq 1$ as the dual ball, $B_{\beta^*}$ is the $\ell_\infty$ ball and contains all other $\ell_p$ balls. This approach relies on mapping $\ell_\infty$ balls through both the forward and backward pass of a network, which zonotopes are well-suited for.

\paragraph{Vector-Jacobian Products of Neural Networks:}
To handle the step of overapproximating the set of vector-Jacobian products, we turn our attention to the structure of the functions we consider. We consider feedforward neural networks with either fully-connected or convolutional layers and elementwise nonlinearities, $\sigma$, such as the ReLU, tanh, or sigmoid operators. A neural network $f$ with $L$ hidden layers may be evaluated according to the following recursion, 
\vspace{-0.5em}
\begin{align}\label{eq:forward-recursion}
    f(x) &:= W_LZ_L(x) + b_L\\ 
    Z_i(x) &= \sigma\Big(\hat{Z_i}(x)\Big)\\
    \hat{Z_i}(x) &= W_iZ_{i-1}(x) + b_i\\
    Z_0(x) & = x,
\end{align}
for $i$ in $\{1,\dots L\}$. When $f$ has convolutional layers, the affine operator in the definition of $\hat{Z_i}(x)$ may be instantiated as a convolution operator.

While the full Jacobian may be expensive to compute, vector-Jacobian products are a standard operation performed on neural nets via the backpropagation algorithm. This may be evaluated according to the following recursion,
\begin{align}
    \nabla_x f(x)^Tu &:= W_1^T J_1(x)^T Y_1(x, u)\\ 
    Y_i(x, u) &= J_{i+1}^T(x) \hat{Y_{i}}(x, u) \\ 
    \hat{Y_i}(x, u) &= W_{i+1}^TY_{i+1}(x, u)\\
    Y_L(x, u) &= u,
\end{align}
where $J_{i}(x)$ is the Jacobian of the $i^{th}$ nonlinearity with respect to its input, $\nabla_{\hat{Z}_i} Z_i(x)$, and $i$ ranges from $1$ to $L-1$. For the standard elementwise nonlinearities, $J_{i}(x)$ is diagonal and may be written as the Hadamard product with a vector. For example, when $\sigma$ is the ReLU operator, this is a Hadamard product with a vector taking entries in $\{0, 1\}$, depending on the sign of the input to each neuron. When convolutional layers are used in place of fully-connected layers, the transpose convolution operator with no bias terms may be used in place of $W_{i+1}^T$ in the definition of $\hat{Y_i}(x,u)$. For generative models that yield images, the outputs are constrained to the hyperbox $[0,1]^n$, usually by applying a sigmoid or tanh layer to the output of $f$. In this case, $\hat{Y}_{L-1}(x, u)$ is $W_L^T J_L^T(x) u$, for $J_L(x)$ denoting the Jacobian with respect to this final nonlinear layer.

As our goal is to generate a set $\mathcal{Y}$ satisfying the containment in equation \ref{eq:main-containment}, we can unroll the recursions and iteratively construct sets which serve as sound approximations of each element in the recursion. For an input range of $\mathcal{X}$, our algorithm will yield a collection of sets $\mathcal{Z}_i, \hat{\mathcal{Z}}_i, \mathcal{J}_i, \mathcal{Y}_i, \hat{\mathcal{Y}}_i$ satisfying the containments 
\vspace{-0.5em}
\begin{align}
    \mathcal{X} &\subseteq \mathcal{Z}_0 &     B_{\beta^*} &\subseteq \hatcal{Y}_L \label{eq:setrec-inputs}\\
    W_i\mathcal{Z}_{i-1} + b_i &\subseteq \hatcal{Z}_i  &   W^T_i \hatcal{Y}_i &\subseteq \mathcal{Y}_{i-1} \label{eq:setrec-aff}\\
    \sigma(\hatcal{Z}_i) &\subseteq \mathcal{Z}_i &   \mathcal{J}_i \odot \mathcal{Y}_{i+1} &\subseteq \hatcal{Y}_i \label{eq:setrec-nonlin}\\
       \nabla_z \sigma(\mathcal{Z}_i) &\subseteq \mathcal{J}_i \label{eq:setrec-jac}&&
\end{align}

That is, we first overapproximate the range of attainable values of each layer of the neural net, $Z_i(\mathcal{X})$ and $\hat{Z}_i(\mathcal{X})$. This allows us to create sets that contain the true range of gradients for each nonlinearity, $\nabla_{Z_i} \hat{Z}_i(\mathcal{X})$ as per the left column. Then a similar procedure is used to obtain sets which contain the true range of partial vector-Jacobian products as the backpropagation algorithm is performed, i.e. $Y_i(\mathcal{X}, B_{\beta^*})$.  Ultimately this will yield a a set that contains the set of attainable vector-Jacobian products $\nabla_x f(\mathcal{X})^T B_{\beta^*}$. Soundness is encapsulated in the following theorem. 
\begin{restatable}{theorem}{soundnessThm}\label{thm:main-soundness-thm}
For feedforward neural networks $f$, an input set $\mathcal{X}$ and sets $\mathcal{Z}_i, \hat{\mathcal{Z}}_i, \mathcal{J}_i, \mathcal{Y}_i, \hat{\mathcal{Y}_i}$ satisfying the containments in Equations \ref{eq:setrec-inputs}-\ref{eq:setrec-jac}, the set of vector-Jacobian products satisfies 
\begin{equation}\label{eq:final-containment}
    \{\nabla_x f(x)^T u\mid x\in\mathcal{X}\quad u\in B_{\beta^*}\} \subseteq \mathcal{Y}_0.
\end{equation}
For such a $\mathcal{Y}_0$, the Lipschitz constant of $f$ may be upper-bounded by  maximizing the $\norm{\cdot}_{\alpha^*}$ norm over the set $\mathcal{Y}_0$.
\end{restatable}

Abstracting this slightly, we notice that each of the recursive containments follow one of four forms. The first is the mapping of sets through affine operators as in equations \ref{eq:setrec-aff}. Next we require the mapping of sets through the nonlinearity $\sigma$ or an elementwise multiplication as in equation \ref{eq:setrec-nonlin}. Third we have the Jacobian operator of $\sigma$ as in equation \ref{eq:setrec-jac}. In the sequel we will describe a family of sets that trades off expressiveness with efficiency of representation, and then we develop a technique to perform each of these four operations in a way that satisfies the required containments. 

\section{Hyperboxes and Zonotopes}
The key idea to handle the sound approximations required by Theorem \ref{thm:main-soundness-thm} is to introduce a family of sets and develop transformations that are closed under this family of sets and preserve the necessary containment for the four classes of operators. The families of sets we will consider here are hyperboxes and zonotopes. Throughout our Lipschitz estimation procedure, we will frequently make use of the fact that linear programs and Minkowski sums of these sets are efficiently computable. 

\paragraph{Hyperboxes:} An axis-aligned hyperbox in $\R^d$ may be defined by a center point $c \in \R^d$ and a radius vector $r\in R^d$ with $r\geq 0$, such that the set $H(c,r)$ may be defined as 
\[H(c,r) := \{c + r\odot y \mid \norm{y}_\infty \leq 1\}.\] 
 Hyperboxes have very efficient representations, and enjoy many nice computational properties. Namely, linear programs over $H(c,r)$ admit a closed-form solution computable in time $O(d)$ as do Minkowski sums:
\begin{align*}
    \max_{x\in H(c,r)} a^Tx &= a^Tc + \norm{a \odot r}_1\\
    H(c_1, r_1) \oplus H(c_2, r_2) &= H(c_1+c_2, r_1+r_2).
    \end{align*}

\paragraph{Zonotopes:}
Zonotopes are a family of sets that can be much more expressive than hyperboxes but also enjoy many of the same efficient subroutines. A zonotope may be defined as the image of an affine operator applied to a hyperbox, or equivalently, a Minkowski sum of line segments. Typically a zonotope in $\R^d$ is represented in the G-representation, where a center $c \in \R^d$ and a generator matrix $G\in \R^{d\times m}$ are supplied. The number of columns, $m$ of $E$ is referred to as the degrees of freedom of a zonotope. Formally, these sets are described as 
\[Z(c,E) := \{c+Ey\mid \norm{y}_\infty \leq 1\}.\]
Linear programs over zonotopes also admit a closed form solution as 
\[\max_{x\in Z(c,E)} a^Tx = a^Tc + \norm{E^Ta}_1\]
which follows from the definitions of dual norms. An important and frequently used application of this fact is that coordinate-wise lower and upper bounds may be efficiently computed. Indeed, the smallest hyperbox containing a zonotope may be written as 
\[Z(c,E) \subseteq H(c, |E|\vec{1}).\]
where $|E|$ denotes the absolute value operator applied elementwise to $E$.
Similarly, the Minkowski sum of two zonotopes is efficiently computable as
\[Z(c_1, E_1) \oplus Z(c_2, E_2) = Z(c_1+c_2, E_1 || E_2)\] 
where $E_1||E_2$ is the concatenation of the columns of $E_1$ with those of $E_2$.

\section{Zonotopes and Sound Pushforward Operators}
Now  we describe how to construct sound transformations as required by the operations outlined in Equations \ref{eq:setrec-inputs}-\ref{eq:setrec-jac}.
\paragraph{Affine operators:}
Both hyperboxes and zonotopes have efficient sound transformations when mapping through affine operators. Consider an affine operator $x\to Ax+b$. In general, hyperboxes are not closed under affine operation, but the tightest sound operator maps a hyperbox $H(c,r)$ to the hyperbox $H(Ac+b,|A|r)$. Zonotopes, on the other hand, are closed under affine operation and $Z(c,r)$ maps to $Z(Ac+b, AE)$. 

\paragraph{Elementwise nonlinearities:}
In general, zonotopes may not be closed under elementwise nonlinearities. Here we will demonstrate a strategy for these mappings that are optimal in a sense and improve upon the mappings of zonotopes through elementwise nonlinearities from prior works.

In general, the problem we consider is to map a zonotope $Z\subseteq \R^d$ through an operator $\Phi :\R^d \to \R^d$ where $\Phi(x) = (\phi(x_1), \dots, \phi(x_d))$. That is, we wish to develop a zonotope $Z'$ satisfying $\Phi(Z)\subseteq Z'$.

The strategy we employ to construct $Z'$ is to retain the structure of $Z$ and incorporate new degrees of freedom. We will scale $Z$ along each coordinate and cover the errors by taking the Minkowski sum with a new zonotope: 
\[Z' = (\Lambda\odot Z) \oplus \hat{Z}\]
where $\Lambda$ is a vector and $\hat{Z}$ represents the a zonotope containing the new degrees of freedom. The following sufficient condition states that this transformation satisfies the desired containment property: 
\begin{restatable}{lemma}{zonoElwise}\label{lemma:zono-containment}
For any zonotope $Z \subset \R^d$ and any operator $\Phi$ operating over $\R^d$, if $\hat{Z}$ is a zonotope satisfying the containment
\begin{equation} \label{eq:soundness-1}
\{\Phi(z) - \Lambda \odot z \mid z\in Z\} \subseteq \hat{Z}
\end{equation}
then 
\[\Phi(Z) \subseteq (\Lambda \odot Z) \oplus \hat{Z}.\]
\end{restatable}
We refer to the set $\{\Phi(z) - \Lambda \odot z \mid z\in Z\}$ as the \emph{residuals}, and reduce the problem to finding a vector $\Lambda$ and set $\hat{Z}$ containing these residuals. Our strategy is to consider sets $\hat{Z}$ that are axis-aligned hyperboxes, i.e. $\hat{Z} = H(b^*, \mu^*)$. While there are many such residual hyperboxes, a reasonable heuristic would be to choose $\Lambda$ to yield the smallest hyperbox satisfying Lemma \ref{lemma:zono-containment}. 

By considering only transformations that scale each coordinate of $Z$ independently and accounting for the residuals with a hyperbox, it suffices to consider each coordinate individually. In this case, the soundness criterion of Equation \ref{eq:soundness-1} reduces to the condition:
\[ \{\phi(z_i) - \Lambda_i z_i \mid z \in Z\} \subseteq [b_i-\mu, b_i+\mu].\]
By our heuristic, we would like to minimize the size of the residual interval, $2\mu$ in the above equation. This may be written as a min-max optimization:
\begin{equation}\label{eq:minmax1}
\min_{\Lambda_i, b_i} \max\limits_{z\in Z} |\phi(z_i) - \Lambda_i z_i - b_i|.
\end{equation}

Indeed, this may be equivalently be viewed as fitting an affine function $L(z_i):= \Lambda_i z_i + b_i$ to the function $\phi(z_i)$ such that the maximum absolute value deviation between $L(z_i)$ and $\phi(z_i)$ is minimized across all $Z$. Now assume that the optimal objective value of the above min-max is $2\mu^*$, and the argmin and argmax are $(\Lambda_i^*, b_i^*)$ and $(z_i^*)$ respectively. Then by definition, we have that
\[\{\phi(z_i) - \Lambda_i^* z_i \mid z\in Z\} \subseteq [b_i^* -\mu^*, b_i^*+\mu^*].\]
By computing the optimal solution to Equation \ref{eq:minmax1} for each coordinate $i$, we can compute $\Lambda^*$ and a residual hyperbox $H(b^*, \mu^*)$ satisfying the sufficient condition required by Lemma \ref{lemma:zono-containment}.

It remains to be seen how to efficiently solve the min-max of Equation \ref{eq:minmax1}. We consider this problem graphically and notice that any 2-dimensional set of points described as the points with vertical deviation of no more than $\mu$ from an affine function $L(z_i)$ is a parallelogram with vertical sides. 
Indeed, we refer to sets of the form 
\[\{(z_i, y_i) \mid z_i \in [l_i, u_i]\quad |y_i-L(z_i)| \leq \mu\}\]
as \emph{vertical parallelograms}, parameterized by the line $L(\cdot)$ and vertical range $\mu$ and denoted as $P(L, \mu)$. As we have shown, the min-max can be reduced to an instance of the vertical-parallelogram fitting problem.

\begin{definition}
The \textbf{vertical parallelogram fitting problem} asks the following question. Given a 2-dimensional set $S$, we seek to find the vertical parallelogram with minimal area that contains the set $S$. 
\end{definition}
Since the horizontal range of the provided set $S$ is fixed, the area of a vertical parallelogram hinges only upon the length of its vertical side. We discuss how to solve this problem in the next section.

Applying this problem to the particular form of the min-max in Equation \ref{eq:minmax1}, we arrive at the following theorem: 
\begin{restatable}{theorem}{setThmOne}
When $S$ is the set $\{(z_i, \phi(z_i)) \mid  z_i \in [l_i, u_i]\}$, the solution to the vertical parallelogram fitting problem yields the optimal solution to Equation \ref{eq:minmax1}. Repeated calls to this subroutine yields the tightest hyperbox fitting the residuals as in equation $\ref{eq:soundness-1}$.
\end{restatable}

\paragraph{Jacobians of Elementwise Operators:}
Now we consider sound operators for the Jacobians of these elementwise nonlinearities, which will ultimately yield the sets $\mathcal{J}_i$ as in equation \ref{eq:setrec-jac}. Specifically the goal is to develop a set containing \[\{\nabla_z \Phi(z) \mid z \in Z \}\]
for any zonotope $Z$. While the strategy presented above certainly applies to this case, we choose to develop a hyperbox approximation for this containment. As hyperboxes allow for independence of coordinates and the $\Phi$ operator is an elementwise operator, this reduces to computing, for each coordinate $i$, the values: 
\[ j_i^{(l)} := \min_{z\in Z} \phi'(z_i) \quad\quad\quad j_i^{(u)}:=\max_{z\in Z} \phi'(z_i).\]
When the $i^{th}$ coordinate of $Z$ is bounded in $[l_i, u_i]$, the above minimum and maximum may be solved efficiently for common nonlinearities. Indeed for ReLU, we have that $j_i^{(l)}=\texttt{sign}(l_i)$ and $j_i^{(u)}=\texttt{sign}(u_i)$. For the sigmoid and tanh operators, this is $j_i^{(l)} = \phi'(\max(\{|l_i|, |u_i|\}))$ and $j_i^{(u)} = \phi'(\min(\{|l_i|, |u_i|\})$. This may be computed for every coordinate $i$ and yields the hyperbox with center $\frac{j^{(u)} + j^{(l)}}{2}$ and radius $\frac{j^{(u)} -j^{(l)}}{2}).$

\paragraph{Elementwise multiplication:}
Using the above machinery, we can handle the elementwise multiplication operator in a sound fashion. Given a zonotope $Z$ and a hyperbox $H$, we wish to develop a zonotope which contains the set 
\[\{x \odot z \mid x\in H,\quad z\in Z\} \]
Parallel to Lemma \ref{lemma:zono-containment}, we employ a strategy where we seek to find the hyperbox that most tightly fits the residual set. This soundness criterion is proved in the following lemma: 
\begin{restatable}{lemma}{zonoSetwise}\label{lemma:soundness-2}
For any zonotope $Z\subseteq \R^d$ and hyperbox $H\subseteq \R^d$, if $\hat{Z}$ is a zonotope satisfying the containment 
\begin{equation}\label{eq:soundness-2}
\{x \odot z - \Lambda \odot z \mid x\in H\quad z \in Z\} \subseteq \hat{Z}
\end{equation}
then
\[\{x \odot z \mid x\in H \quad z\in Z\} \subseteq (\Lambda \odot Z) \oplus \hat{Z}.\]
\end{restatable}{lemma}

Since $\odot$ acts elementwise, $H$ is a hyperbox, and we only seek to fit a hyperbox residual, we may again consider each coordinate independently. This reduces to another min-max problem where the maximum is now taken over both $Z$ and $H$.
\begin{equation}\label{eq:minmax2}
\min\limits_{\Lambda_i, b_i} \max\limits_{z\in Z, x\in H} |x_i \cdot z_i - \Lambda_i z_i - b_i|.
\end{equation}
We may again solve this via a reduction vertical-parallelogram fitting problem. We can suppose that $z_i$ is contained in the interval $[l^{(z)}_i, u^{(z)}_i]$ and $x_i$ is contained in the interval $[l^{(x)}_i, u^{(x)}_i]$. Then the following theorem relates the vertical parallelogram fitting problem to the elementwise multiplication operator.
\begin{restatable}{theorem}{setThmTwo}
When $S$ is the set $\{(z, x \cdot z) \mid l^{(z)} \leq z \leq u^{(z)} \quad  l^{(x)} \leq x \leq u^{(x)} \}$, the solution to the vertical-parallelogram fitting problem yields the optimal solution to Equation \ref{eq:minmax2}. Repeated calls to this subroutine yields the tightest hyperbox fitting the residuals as in Equation \ref{eq:soundness-2}.
\end{restatable}{theorem}
In this sense, we may once again compute the vertical parallelogram fit for each coordinate $i$ to generate the scaling factor $\Lambda$ and residual hyperbox $Z'.$

\section{Vertical-Parallelogram fitting problem}\label{sec-vpfit}
\begin{figure*}[ht]
    \centering
    \includegraphics[width=0.34\textwidth]{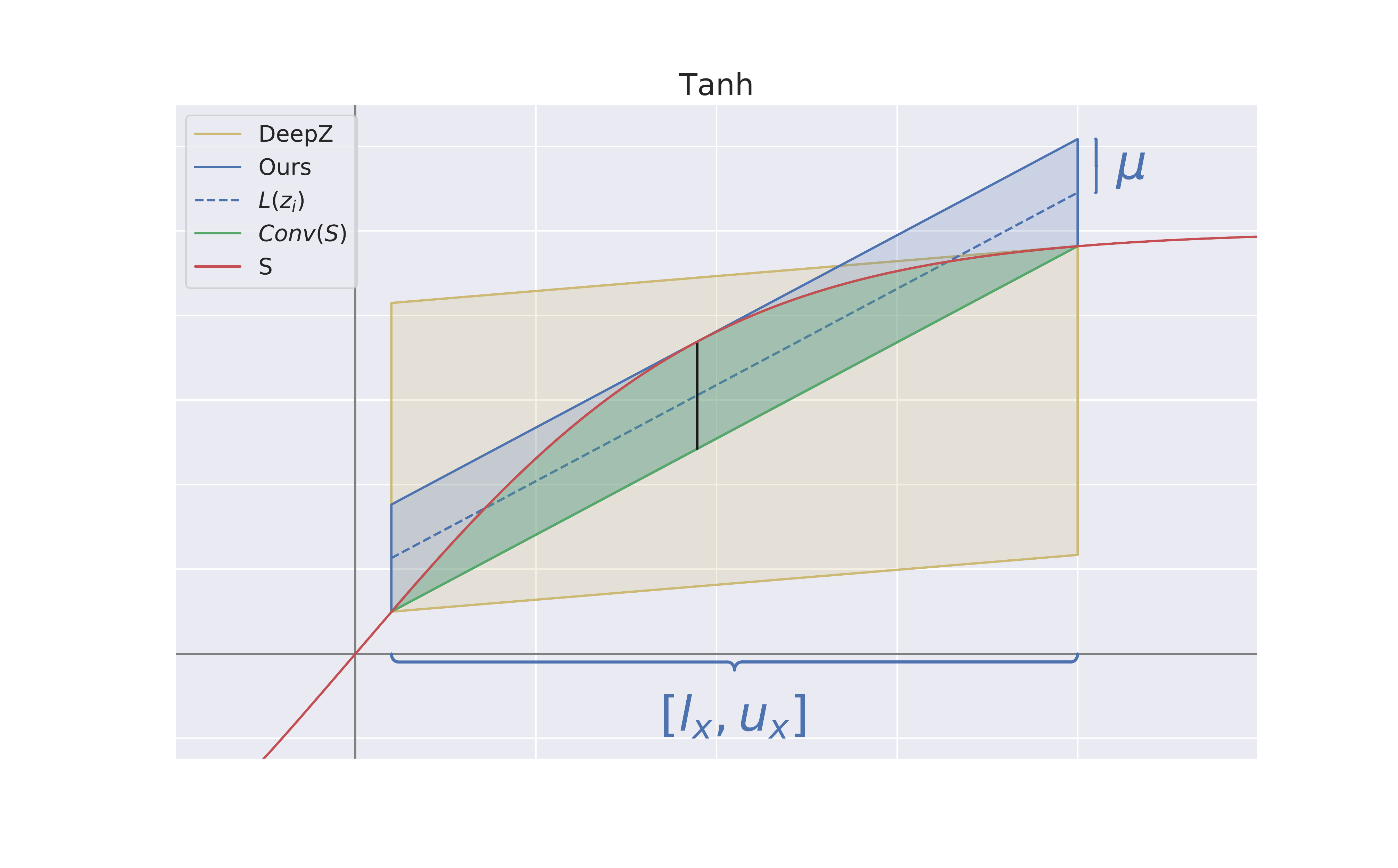}
    \includegraphics[width=0.34\textwidth]{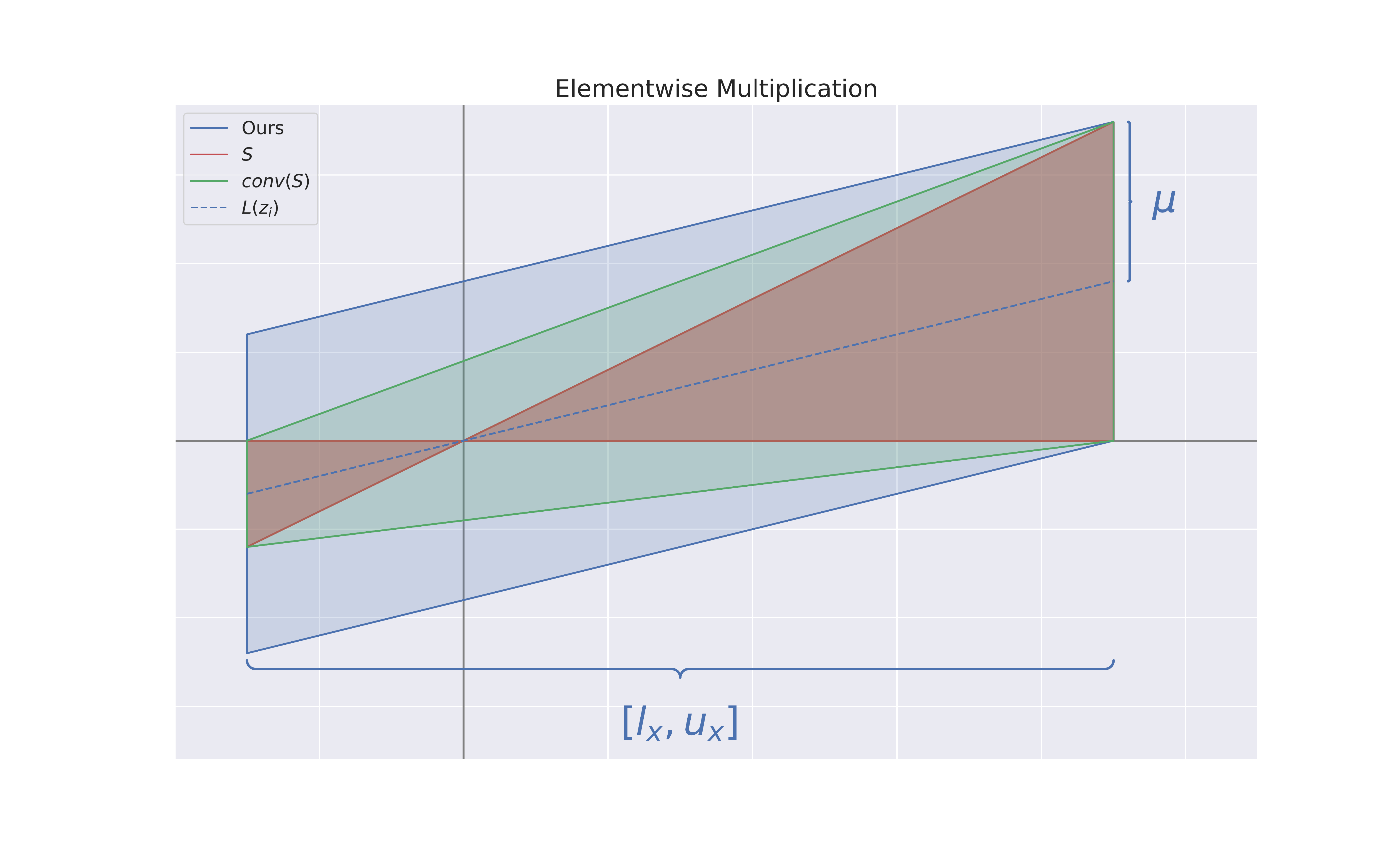}
    \includegraphics[width=0.24\textwidth]{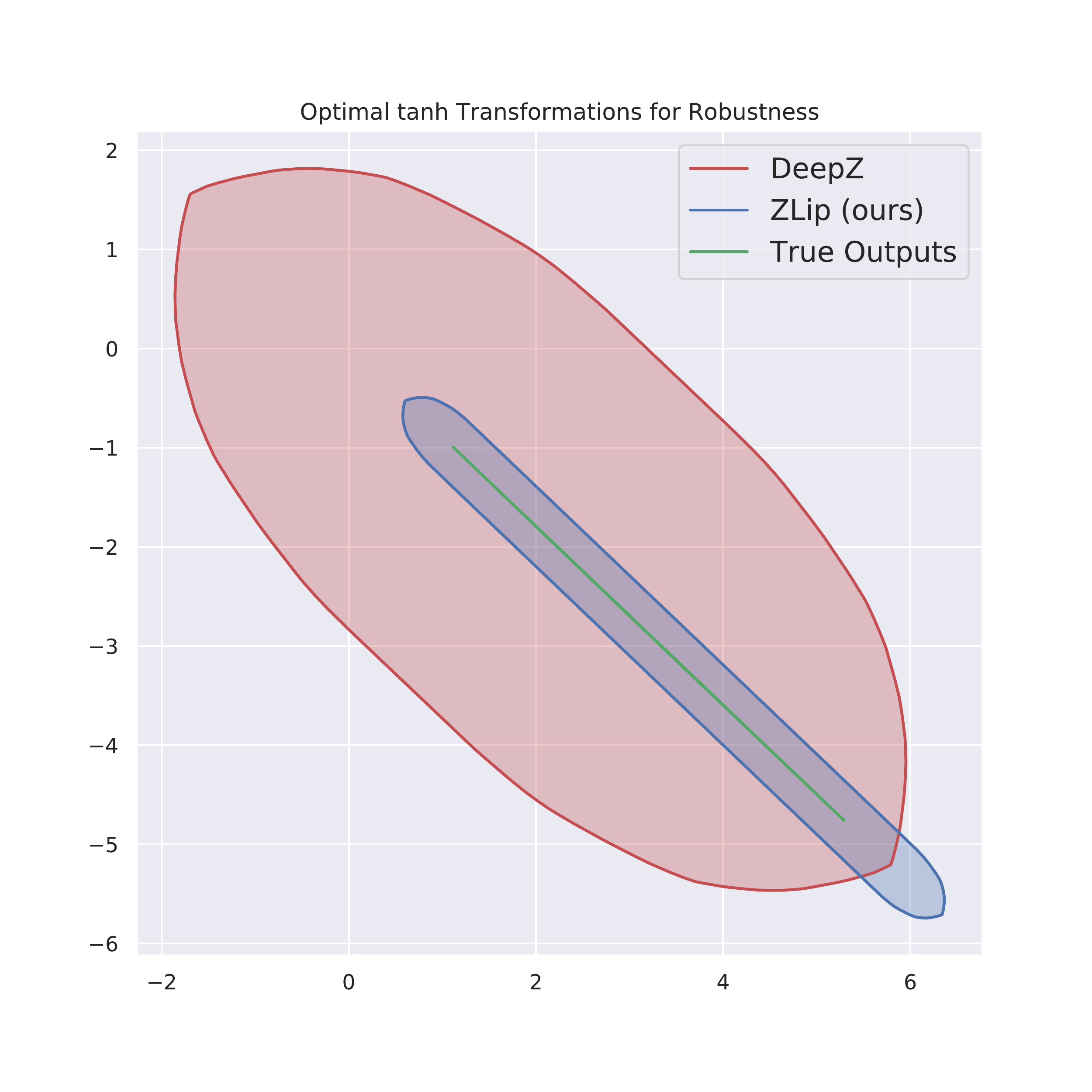}\\
    \caption{Examples of the vertical parallelogram fitting algorithm for the tanh (left) and elementwise multiplication (middle) operators. The sets $S$ are drawn in red, with their convex hulls in green. The optimal vertical parallelogram for each is drawn in blue. The bound yielded by DeepZ is in yellow for the tanh operator. For classifiers trained on the toy dataset, our improved tanh bounds yield tighter reachability analysis compared to prior work (right). The set of attainable outputs of a network for a specified region is plotted in green, DeepZ yields the set shaded in red, whereas our reachable set is shaded in blue.
    }\label{fig:tanh-elmul}
    \vspace{-1em}
\end{figure*}

In the cases of elementwise operators or elementwise multiplication by a hyperbox, we have reduced the problem of tightly fitting the residuals to the vertical parallelogram fitting problem. Here we describe our algorithm to solve this problem and illustrate its use on two examples. Prior work has optimally solved this problem for the ReLU nonlinearity, but is unnecessarily loose for differentiable nonlinearities. We provide full derivations for the ReLU, sigmoid, tanh, and absolute value operators in the appendix.

\paragraph{Algorithm} 
Consider some general 2-dimensional set $S$. We will walk through the major components of our algorithm and describe the necessary subroutines to fit a vertical parallelogram to $S$. Recall that vertical-parallelograms are parameterized by a scalar vertical side-length, $2\mu$, and an affine function $L(\cdot)$ which is parameterized by a slope $\lambda$ and intercept $b$.

First we compute the vertical-side length $\mu$. Since vertical parallelograms are convex, any vertical parallelogram containing the target set $S$ must contain its convex hull. Assuming that the set $S$ is bounded in the horizontal dimension by $[l_x, u_x]$, the convex hull of $S$ may be decomposed into a convex upper and lower hull, $h^+(x), h^-(x)$ such that 
\[\texttt{conv}(S) = \{(x, y) \mid x\in [l_x, u_x],\quad h^-(x) \leq y \leq h^+(x)\},\]
where the upper hull is concave and the lower hull is convex. The vertical deviation between these hulls, which we refer to as the altitude at $x$, is denoted by their difference $h^+(x)-h^-(x)$. The vertical side-length $2\mu$ is lower bounded by the altitude at $x$ for every $x\in [l_x, u_x]$. If the maximum of the altitude across $[l_x, u_x]$ is $2\mu^*$, attained at $x^*$, the vertical side-length of $P$ must satisfy $2\mu^* \leq 2\mu$. It turns out this is always attainable with equality, i.e. there always exists an affine function $L$ such that $P(L, \mu^*)$ contains $S$.

For maximum altitude $2\mu^*$ attained at $x^*$, we compute an affine function $L$ such that the shifted affine function $L(x)+\mu^*$ is greater than $h^+(x)$ for all $x\in [l_x,u_x]$ and vice versa for the lower convex hull, $L(x)-\mu^*\leq h^-(x)$. We first solve for the slope of $L$. The subgradient of a convex function $f$, denoted $\delta f(x)$, is the set of slopes of affine functions passing through the point $(x, f(x))$ that lower bound $f$. Thus, the slope $\lambda$ must lie in both the subgradient of $h^-(x)$ and $-h^+(x)$. The intersection of these subgradients is guaranteed to be nonempty as $x^*$ minimizes the convex function $h^-(x)-h^+(x)$, implying $0\in \delta(h^-(x^*)-h^+(x^*)).$ Any element in the intersection of these subgradients suffices as a choice for the slope of the fitting parallelogram. We arbitrarily choose such an element and denote it as $\lambda^*$. 

Finally we compute the intercept of the line $L$. For the bounds to be tight, the constraint that $L(x)+\mu^* \geq h^+(x)$ must be tight at $x^*$ and likewise for $h^-$. Hence the line $L$ must pass through the point $(x^*, \frac{h^+(x^*)+h^-(x^*)}{2})$. Provided with a slope $\lambda^*$, it is trivial to compute the intercept of a line passing through that point. This concludes the algorithm and yields the desired parallelogram attaining the tightest possible vertical deviation. 

The running time of this algorithm depends crucially on the structure of the set $S$. Computing the upper and lower convex hulls of $S$ may be challenging and algorithms to compute these functions need to be designed by hand. However, the sets we consider admit efficient algorithms to compute these functions. Multiple prior works have correctly computed the convex hull for the set defined by the ReLU operator, but have provided suboptimal bounds for differentiable nonlinearities. To illustrate its use, we will apply our algorithm to the tanh operator and the elementwise multiplication operator. 

\paragraph{Tanh example: } 
Consider the $\tanh$ function and the horizontal coordinate bounds $[l_x, u_x]$ where $l_x \geq 0$. In this case, we desire to fit a parallelogram to the set $\{(x, \tanh(x)) \mid x \in [l_x, u_x]\}$. The $\tanh$ function is concave for $x \geq 0$, so the convex hull of $S$ has an upper hull of $h^+(x)=\tanh(x)$. The lower hull, $h^-(x)$, is the secant line passing through the points $(l_x, \tanh(l_x))$ and $(u_x, \tanh(u_x)$, which has slope $\lambda = \frac{\tanh(u_x)-\tanh(l_x)}{u_x-l_x}$. The altitude is maximized when the slope of $\tanh(x)$ is equal to $\lambda$, which is attained at $x^*=\tanh^{-1}(\sqrt{1 - \lambda})$. This yields the maximal vertical deviation $2\mu^* = \tanh(x^*)-h^-(x^*)$. Since the lower hull is a secant line, the subgradients everywhere are $\lambda$, so the slope of $L(\cdot)$ must be $\lambda$. The intercept of the parallelogram's line $L(\cdot)$ must pass through the point $(x^*, \tanh(x^*)-\mu^*).$ This is illustrated graphically in figure \ref{fig:tanh-elmul} (left), where the set $S$ is drawn in red and its convex hull is green. The parallelogram yielded by our algorithm is in blue, whereas the parallelogram yielded by prior work is yellow \cite{Singh2018-tg}. 

\paragraph{Elementwise multiplication example:} 
Now consider the case of elementwise multiplication. Consider the $x$-interval with $l_x < 0 < u_x$ multiplied by a value $y$ in the range $[\alpha, \beta]$ for $ \alpha \geq 0$. The set $S=\{(x, x\cdot y) \mid l_x \leq x \leq u_x\quad \alpha \leq y \leq \beta\}$ is the union of the two triangles shown in red in Figure \ref{fig:tanh-elmul} (middle). The convex hull of $S$ is the trapezoid, in green in Figure \ref{fig:tanh-elmul}, and the upper hull is the line connecting the points $(l_x, \alpha l_x)$ and $(u_x, \alpha u_x)$. The lower hull is the line connecting the points $(l_x, \beta l_x)$ and $(u_x, \alpha u_x)$. The maximum altitude is attained at one of the endpoints: in the example in Figure \ref{fig:tanh-elmul} this is $u_x$ because $u_x>|l_x|$. The altitude is then $(\alpha-\beta)\cdot u_x$. The slope $\lambda$ must lie between the slopes of the directional derivatives of the upper and lower hulls, yielding an admissible range of $[\frac{\alpha u_x - \beta l_x}{u_x-l_x}, \frac{\beta u_x-\alpha l_x}{u_x-l_x}]$. We choose $\lambda$ to be the midpoint of this interval. The intercept of the center line may then be chosen such that it passes through the point $(u_x, \frac{\alpha+\beta}{2} u_x)$. We plot the resulting parallelogram in blue.

\section{Maximizing Norms over Zonotopes}
The final step to upper bounding the Lipschitz constant of a network is to compute a maximization of the $\norm{\cdot}_{\alpha^*}$ norm over a zonotope $Z$, which contains the set of all attainable vector-Jacobian products. While maximizing a convex norm over a convex set may be hard in general, it suffices to upper bound this value. We may always upper bound this norm efficiently by transforming $Z$ into the tightest containing hyperbox and computing the norm over this hyperbox. Any maximal-$\ell_p$ norm of a hyperbox is efficiently computable, so this technique is quite efficient. However as we typically consider the $\alpha$ to be the $\ell_\infty$ norm, we focus on techniques to maximize the dual $\ell_1$ norm specifically. First we show that this problem is equivalent to computing the Grothendieck problem, i.e. to compute the matrix norm $\norm{\cdot}_{\infty \to 1}.$
\begin{restatable}{theorem}{grothendieckThm}
The problem of computing the maximal $\ell_1$ norm of a zonotope is equivalent to the $\norm{\cdot}_{\infty \to 1}$ matrix norm: both problems are NP-hard in general. Additionally, any approximation algorithm with approximation ratio $\alpha$ for the Grothendieck problem will yield an approximation algorithm with ratio $\alpha$ for the zonotope $\ell_1$ maximization problem and vice versa.
\end{restatable}{theorem}
The Grothendieck problem is well-studied and it has been shown that the semidefinite relaxation yields an approximation ratio of $<1.783$ \cite{Braverman2013-vg}. However, this relaxation may be quite slow. We present a novel result that states that the linear-programming relaxation for the $\ell_1$ zonotope norm maximization problem, and equivalently the Grothendieck problem, may be computed in linear time. 
\begin{restatable}{theorem}{lpGrothendieck}\label{thm:lp-relax}
For a zonotope, $Z(c,E)$, the linear programming relaxation of $\max_{z\in Z(c, E)} \norm{z}_1$ is computable in time $O(|E|)$ where $|E|$ denotes the number of elements in $E$.
\end{restatable}{theorem}
The proof follows from applying the vertical-parallelogram fitting algorithm to the absolute value operator and then solving a linear program over the resulting zonotope. 

\section{Experiments}

\begin{table*}[]

\centering
\caption{Evaluation of ZLip and Fast-Lip on generative models trained on the MNIST and CIFAR-10 datasets, evaluated over inputs of radius 0.05. Times are reported in seconds The column F/Z denotes the ratio of the estimate returned by Fast-Lip to our estimate. For larger networks, ZLip can yield upper bounds that are several orders of magnitude tighter than Fast-Lip.}\label{table:generator-table}
\begin{tabular}{@{}l|llll|l|ll|ll|l@{}}\midrule
                             & \multicolumn{5}{c|}{MNIST}                                                             & \multicolumn{5}{c}{CIFAR-10}                                                             \\ \midrule
\multicolumn{1}{c|}{Network} & \multicolumn{2}{c|}{ZLip}              & \multicolumn{2}{c}{Fast-Lip} &                & \multicolumn{1}{c}{ZLip} &      & \multicolumn{1}{c}{Fast-Lip} &        &                 \\ \midrule
                             & Value     & \multicolumn{1}{l|}{Time}  & Value          & Time        & F/Z            & Value                    & Time & Value                        & Time   & F/Z             \\ \cmidrule(l){2-11} 
VAESmall                     & 6.81\e{2} & \multicolumn{1}{l|}{0.831} & 6.30\e{3}      & 0.0017      & \textbf{10.29} & 4.07\e{3}                & 4.97 & 1.13\e{4}                    & 0.0029 & \textbf{2.79}   \\
VAEMed                       & 2.39\e{3} & \multicolumn{1}{l|}{1.22}  & 1.06\e{6}      & 0.0029      & \textbf{1024}  & 8.28\e{3}                & 5.99 & 1.06\e{6}                    & 0.0042 & \textbf{133.6}  \\
VAEBig                       & 6.56\e{3} & \multicolumn{1}{l|}{1.38}  & 5.71\e{7}      & 0.0042      & \textbf{46746} & 8.39\e{e}                & 5.41 & 1.74\e{7}                    & 0.0054 & \textbf{2227.8} \\
VAECNN                       & 8.75\e{2} & \multicolumn{1}{l|}{1.37}  & 3.58\e{4}      & 0.0028      & \textbf{47.75} & 5.97\e{3}                & 10.3 & 6.16\e{4}                    & 0.0031 & \textbf{10.38}  \\
FFGAN                        & 1.55\e{7} & \multicolumn{1}{l|}{0.455} & 2.95\e{8}      & 0.0044      & \textbf{52.99} & 1.24\e{7}                & 1.14 & 1.21\e{8}                    & 0.081  & \textbf{10.01}  \\
DCGAN                        & 4.34\e{5} & \multicolumn{1}{l|}{3.46}  & 2.24\e{7}      & 0.0056      & \textbf{55.33} & 2.11\e{6}                & 8.15 & 4.63\e{7}                    & 0.0127 & \textbf{31.19}  \\
VAETanh                      & 3.55\e{3} & \multicolumn{1}{l|}{2.17}  & 7.96\e{4}      & 0.0032      & \textbf{24.19} & 2.26\e{3}                & 3.90 & 2.97\e{5}                    & 0.0044 & \textbf{132.8}  \\
VC-Tanh                      & 2.40\e{3} & \multicolumn{1}{l|}{2.95}  & 7.71\e{4}      & 0.0031      & \textbf{37.25} & 5.76\e{3}                & 1.97 & 1.86\e{5}                    & 0.0033 & \textbf{32.47} 
\end{tabular}

\end{table*}
We highlight that our algorithm, which we refer to as ZLip, is specifically designed to provide Lipschitz estimates of networks with large output dimension. However, the approaches outlined above are applicable to scalar functions as well. As much of the literature focuses on classifiers, we first compare our approach on a binary classification task against other Lipschitz estimation techniques. Then we apply ZLip to generative models for MNIST and CIFAR-10. A full description of the experimental details and additional experiments on 
MNIST and CIFAR-10 classifiers are present in the supplementary. The code is publicly available at \href{https://github.com/revbucket/lipMIP}{https://github.com/revbucket/lipMIP}.

\paragraph{Toy Network Benchmarks:}

\begin{figure*}[ht]
    \centering
    \includegraphics[width=0.48\textwidth]{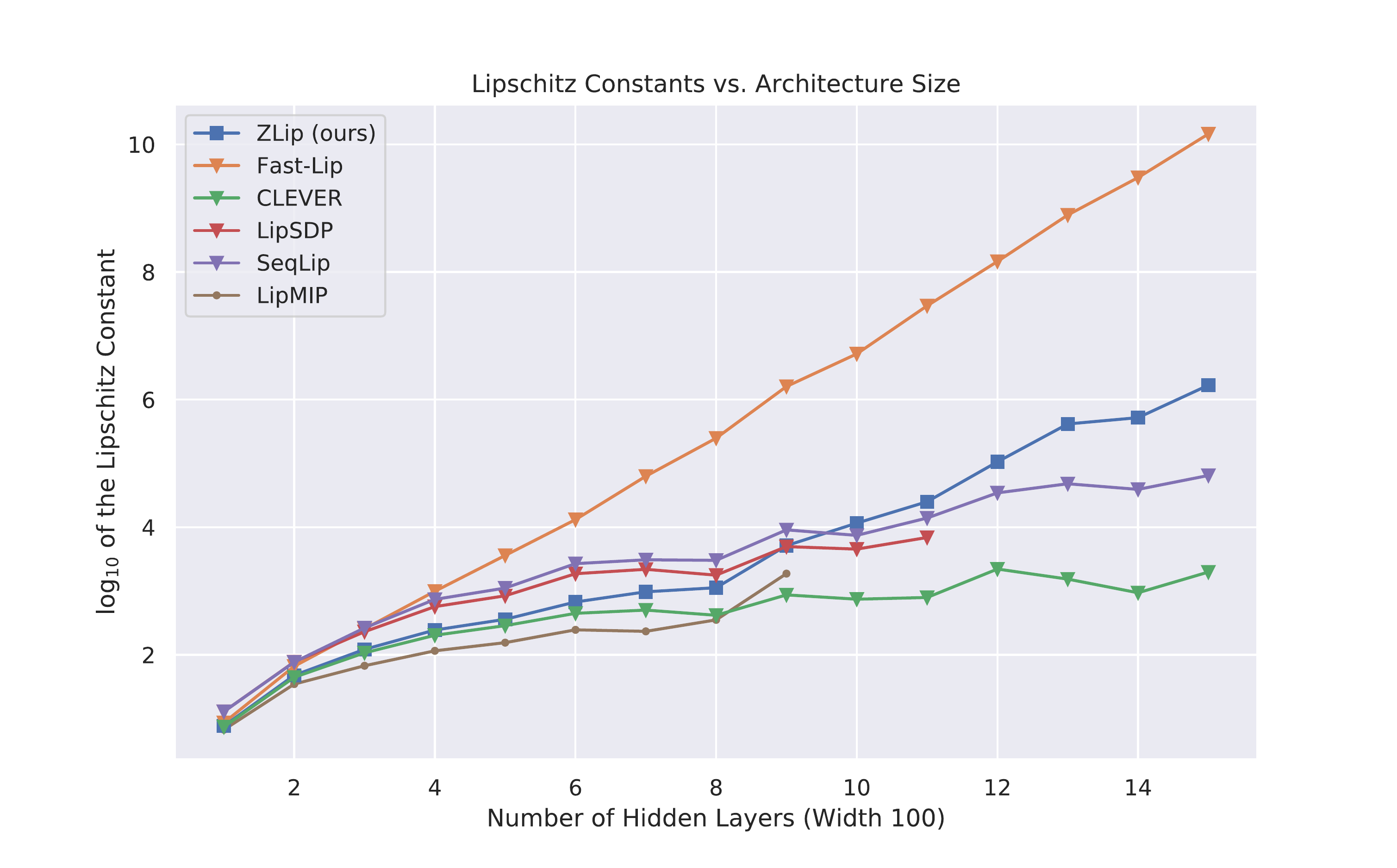} 
    \hfill
    \includegraphics[width=0.48\textwidth]{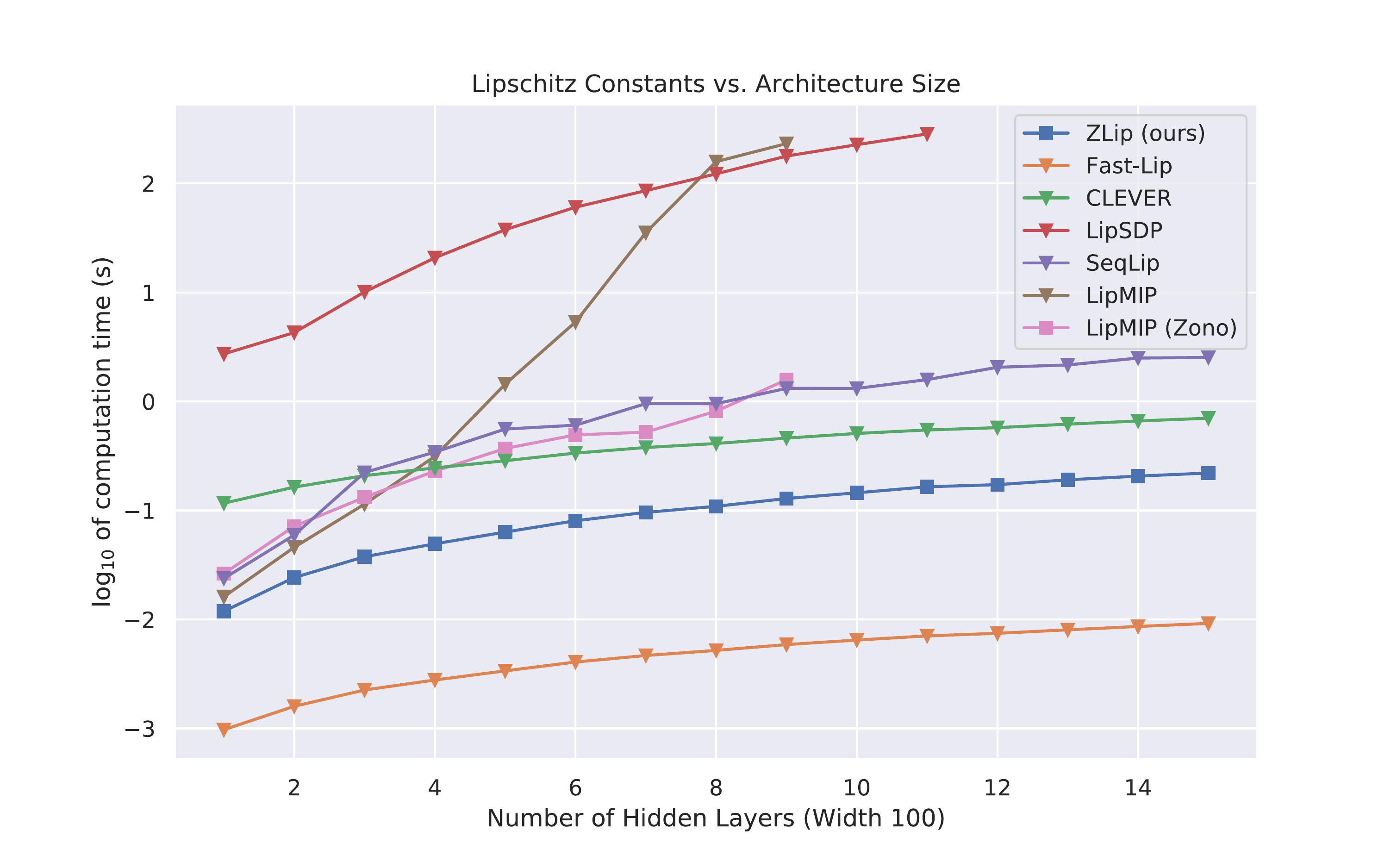}
    \caption{Comparisons of Lipschitz estimate (left) and running time (right) for various Lipschitz estimation techniques on networks trained on the Circle dataset. The horizontal axis is the number of hidden layers of the networks considered and the vertical axis is in log-scale. Our approach is scalable to larger networks while still providing reasonably tight bounds. Also note that using our approach within LipMIP dramatically improves the efficiency.\label{fig:toy-dataset}}
\end{figure*}
To fairly compare against existing Lipschitz estimation techniques, we present results on the 2-dimensional Circle dataset from \cite{Aziznejad2020-aq}, where the binary classification is resolved by taking the sign of the output. We consider the  $L^{(\infty, |\cdot|)}(f, \mathcal{X})$ Lipschitz constant for fully-connected networks $f$ with input dimension 2 and a varying amount of  layers of width 100 and the ReLU nonlinearity. We report the average Lipschitz estimate and compute time for the following techniques: Fast-Lip, LipSDP, SeqLip, CLEVER, and LipMIP. Fast-Lip and LipSDP provide provable upper bounds \cite{Weng2018-gr, Fazlyab2019-im}. SeqLip and CLEVER provide heuristic estimates and LipMIP computes this quantity exactly \cite{Virmaux2018-ti, Weng2018-lf, Jordan2020-tq}. LipMIP leverages interval analysis as a first step, so we also consider a modified version that instead uses the layerwise approximations yielded by ZLip.

In Figure \ref{fig:toy-dataset}, we plot the reported Lipschitz estimate and runtime of these other techniques applied on input regions that are random hyperboxes of $\ell_\infty$ radius $0.1$ centered at elements in the test set. These plots can demonstrate where each technique lies with respect to the efficiency-accuracy tradeoff. In varying the architecture size, we observe that ZLip yields the tightest provable upper bounds for small networks, and only begins to provide looser bounds than LipSDP at 9 hidden layers, at which point LipSDP is three orders of magnitude slower than ZLip. Additionally, using ZLip in place of interval analysis in LipMIP can provide speedups of up to 100x while preserving the exactness of Lipschitz computation. 

\paragraph{Generative Models:}

We now scale our approach to generative models for the MNIST and CIFAR-10 datasets. We train multiple VAEs and GANs using fully-connected and convolutional layers and the ReLU and tanh nonlinearities \cite{Kingma2013-en, Radford2015-hi}. To evaluate over VAEs, we consider input sets $\mathcal{X}$ that are $\ell_\infty$ balls surrounding the encodings of images taken from the test set. For GAN evaluation, we consider $\ell_\infty$ balls surrounding random inputs from the training distribution and evaluate $L^{(\infty, 1)}(f, \mathcal{X})$ of the generator. The only other nontrivial Lipschitz estimation approach that tolerates vector-valued networks and is able to scale to networks of this size is Fast-Lip. Full experimental details are presented in the Appendix, as well as experiments with input sets of different radii.

Table \ref{table:generator-table} displays results for random inputs with $\ell_\infty$ radius of 0.05. We report the average Lipschitz bound and runtime, as well as the ratio of the Fast-Lip value to the ZLip value, denoted F/Z. As the network size scales, we notice that the values reported by both ZLip and Fast-Lip increase. This is likely due to both the true Lipschitz constant of the network increasing as well as the bounds becoming looser. However the bounds provided by ZLip are comparatively much tighter than Fast-Lip for larger networks. We attribute this to the fact that the interval analysis of Fast-Lip introduces error in both the affine and nonlinear layers, whereas our approach only introduces error in the nonlinear layers. While an increase in network size accompanies an increase in runtime for both techniques and ZLip is indeed slower than Fast-Lip, ZLip remains tractable for even the largest networks we consider. 
\section{Conclusion} 
We have presented a principled way to efficiently upper bound the Lipschitz constant of neural networks with large output dimension. Our technique improves upon prior works for zonotope approximations of neural networks and is applicable to the operators required by the backpropagation algorithm. This approach yields tighter provable Lipschitz bounds for classifiers and is able to scale effectively to familiar generative models for the MNIST and CIFAR10 datasets, yielding improvements of up to three orders of magnitude for Lipschitz estimation of these networks. 

\paragraph{Acknowledgements:} This research has been supported by NSF Grants CCF 1763702,1934932, AF 1901292,
2008710, 2019844 research gifts by Western Digital, WNCG IAP, computing resources from TACC and the Archie Straiton Fellowship. 

\bibliography{main}
\bibliographystyle{icml2021}
\newpage


\onecolumn
\appendix
\section{Proofs of Lemmas and Theorems} 

\soundnessThm*
\begin{proof} 
Suppose $\mathcal{Z}_i, \hat{\mathcal{Z}}_i, \mathcal{J}_i, \mathcal{Y}_i, \hat{\mathcal{Y}}_i$ satisfy the containments is equation \ref{eq:setrec-inputs}-\ref{eq:setrec-jac}. Now consider any $x\in\mathcal{X}$ and $u\in B_{\beta^*}$. The proof follows from repeated applications of the following statement: for any function $g$, if $\mathcal{A} \subseteq \mathcal{B}$, then $g(\mathcal{A}) \subseteq g(\mathcal{B})$. We iteratively apply this statement to the forward recursion to see that $\hat{Z}_i(x) \in \hat{\mathcal{Z}}_i$ for all $i$, and similarly for $Z_i(x)\in \mathcal{Z}_i$. From equation \ref{eq:setrec-jac}, $J_{i+1}^T(x)\in  \mathcal{J}_{i+1}$ for all $i$. We may now perform the backward recursion to see that $Y_i(x,u) \in \mathcal{Y}_{i}$ and similarly for $\hat{Y}_i(x,u), \hat{\mathcal{Y}}_i.$ Repeating this for all $i$ yields the desired result. 
\end{proof}

\zonoElwise*
\begin{proof} 
By assumption, for every $z\in Z$, there exists a $\hat{z}\in\hat{Z}$ such that $\Phi(z)-\Lambda \odot z = \hat{z}.$ This implies that, $\Phi(z)=(\Lambda \odot z) + \hat{z}$. By definition, 
\[(\Lambda \odot Z) \oplus \hat{Z} := \{\Lambda \odot z + \hat{z}\mid z\in Z\quad \hat{z}\in \hat{Z}\},\]
so $\Phi(z)\in (\Lambda \odot Z) \oplus \hat{Z}$ for every $z\in Z$.
\end{proof}

\setThmOne*
\begin{proof} 
Suppose that $\{\Lambda_i, b_i, \mu_i\}_{i=1}^d$ is the set of solutions to the vertical parallelogram fitting problem for each set $S_i=\{(z_i, \phi(z_i) \mid z_i\in[l_i, u_i]\}$. Since the coordinate-wise bounds $l_i, u_i$ are chosen such that $l \leq z \leq u$ for all $z\in Z$, the containment holds: 
\[\{\Phi(z) - \Lambda \odot z -b \mid z\in\mathcal{Z}\} \subseteq \{\Phi(z) - \Lambda \odot z -b \mid l \leq z \leq u\}.\] 
By definition of solutions to the vertical parallelogram fitting problem, the set of vectors $\{\Phi(z) - \Lambda \odot z -b \mid l \leq z \leq u\}$ is contained in the hyperbox $H(0, \mu)$. Adding $b$ to each element of each set, we see that $\{\Phi(z) - \Lambda \odot z \mid l\leq z \leq u\}$ is contained in the hyperbox $H(b, \mu)$, thus satisfying the assumptions of Lemma \ref{lemma:zono-containment}.
\end{proof}

\zonoSetwise*
\begin{proof} 
By assumption, for every $z\in Z$ and every $x\in H$, there exists a $\hat{z}\in\hat{Z}$ such that $x\odot z - \Lambda \odot z= \hat{z}$. This implies that, $x \odot z=(\Lambda \odot z) + \hat{z}$. And by definition 
\[(\Lambda \odot Z) \oplus \hat{Z} := \{\Lambda \odot z + \hat{z}\mid z\in Z\quad \hat{z}\in \hat{Z}\},\]
so $x \odot z \in (\Lambda \odot Z) \oplus \hat{Z}$ for every $z\in Z$.
\end{proof} 

\setThmTwo*
\begin{proof}
Suppose that $\{(\Lambda_i, b_i, \mu_i)\}_{i=1}^d$ is the set of solutions to the vertical parallelogram fitting problem for each set $S_i=\{(z_i, x_i\odot z_i) \mid l^{(z)}_i \leq z_i \leq u^{(z)}_i, \quad l^{(x)}_i \leq x_i \leq u^{(x)}_i \}$. Let $H$ be the hyperbox with lower and upper-bounds denoted by $l^{(x)}, u^{(x)}$. Since the coordinate-wise bounds $l_i, u_i$ are chosen such that $l^{(z)} \leq z \leq u^{(z)}$ for all $z\in Z$, the containment holds: 
\[\{x \odot z - \Lambda \odot z -b \mid z\in\mathcal{Z}\quad x\in H\} \subseteq \{\Phi(z) - \Lambda \odot z -b \mid l^{(z)} \leq z \leq u^{(z)} \quad x\in\mathcal{H}\}.\] 
By definition of solutions to the vertical parallelogram fitting problem, the set of vectors $\{x\odot (z) - \Lambda \odot z -b \mid l^{(z)} \leq z \leq u^{(z)}\quad x\in H\}$ is contained in the hyperbox $H(0, \mu)$. Adding $b$ to each element of each set, we see that $\{x\odot z - \Lambda \odot z \mid l\leq z \leq u\quad x\in H\}$ is contained in the hyperbox $H(b, \mu)$, thus satisfying the assumptions of Lemma \ref{lemma:soundness-2}.
\end{proof} 

\grothendieckThm*
\begin{proof} 
We prove this via a strict reduction in showing that any instance of one problem may be converted into an instance of the other and will keep the same optimal value. To do this, we first note that for any matrix $M$, with $(0||M)$ denoting the zero-column prepended to the columns of $M$, that $\norm{M}_{\infty\to 1}=\norm{(0||M)}_{\infty \to 1}$. This follows since 
\[\norm{(0||M)}_{\infty\to1} = \max_{v\in B_{\infty}}\norm{(0||M)v}_1 = max_{u\in B_{\infty}}\norm{Mu}=\norm{M}_{\infty\to 1}.\]

Next, it suffices to show that 
\begin{equation}\label{eq:strict-reduction}
    \max\limits_{z\in Z(c,E)} \norm{z}_1 = \norm{(c||E)}_{\infty\to 1},
\end{equation}
for if this were true, certainly any zonotope could be reduced to a matrix-norm maximization problem, and any matrix norm problem could first prepend the zero column to the matrix and be reduced to a zonotope norm-maximization problem. Any $\alpha$-approximation algorithm for one problem could provide an $\alpha$-approximation for any instance of the other via this reduction. 

First we show that $\max\limits_{z\in Z(c,E)} \norm{z}_1 \leq \norm{(c||E)}_{\infty\to 1}$. As the right-hand-side may be written 
\[ \norm{(c||E)}_{\infty\to 1} = \max\limits_{|v_0|\leq 1} \max\limits_{v\in B_\infty} \norm{v_0\cdot c + Ev}_1\]
and whereas the left-hand side of Equation \ref{eq:strict-reduction} is the same optimization with $v_0$ restricted to $1$. Therefore the $(\leq)$ direction of Equation \ref{eq:strict-reduction} holds. For the other direction, consider any integral solution to the RHS, 
\[ (v_0^*, v^*) \in \argmax\limits_{|v_0|\leq 1, v\in B_\infty} \norm{v_0\cdot c + Ev}_1.\] 
Without loss of generality, $v_0^*$ may be chosen to be 1, and the point $(c+Ev^*)$ is in $Z(c,E)$.  Hence there's a point in $Z(c,E)$ with $\ell_1$ norm at least that of $\norm{(c||E)}_{\infty\to 1}$, thus proving the $(\geq)$ direction of equality \ref{eq:strict-reduction}. 
\end{proof} 

\lpGrothendieck*
\begin{proof}
First we write down the Linear-programming relaxation of the zonotope-norm maximization problem and then relate this to the mapping of the zonotope through the absolute value operator, by our vertical-parallelogram fitting procedure. The final result follows from linear programs being efficiently solvable over zonotopes. 

Consider some zonotope $Z(c, E) \subseteq \R^d$ which has coordinate-wise upper and lower bounds $[l_i, u_i]$ for every $i\in [d].$ We partition the coordinates into three sets of indices: $S^-, S^+, S$ such that $S^-:=\{i \mid u_i \leq 0\}$, $S^+:=\{i \mid l_i >0\}$ and $S$ is the set of indices not in either $S^-$ or $S^+$. We may write down the familiar mixed-integer programming relaxation for the absolute value operator by introducing $|S|$ continuous variables, $\{t_i\}_{i\in S}$, and $d$ integer variables $\{a_i\}_{i\in S}$, where $t_i\in \R$ and $a_i\in \{0,1\}$:
\begin{align}
    \max &\sum_{i\in S} t_i -\sum_{i\in S^-} z_i  +\sum_{i\in S^+} z_i\\
    t_i &\geq z_i \\
    t_i &\geq -z_i \\ 
    t_i &\leq -z_i + 2\cdot u_i \cdot a_i \label{eq:mip-1}\\ 
    t_i &\leq z_i - 2\cdot l_i \cdot (1-a_i) \label{eq:mip-2}\\
    a_i &\in \{0, 1\} \\
    z&\in Z(c, E)
\end{align}
Where the constraints enforce that $t_i=|z_i|$. The first two constraints require that $t_i\geq |z_i|$. To show $t_i\leq |z_i|$, we proceed by cases. When $z_i>0$, then \ref{eq:mip-1} implies that $a_i=1$, for otherwise $t_i < 0$ contradicting the first constraint. This causes \ref{eq:mip-2} to imply $t_i\leq z_i$. When $z_i <0$, \ref{eq:mip-2}, $a_i=0$, for otherwise \ref{eq:mip-1} again implies that $t_i<0$. This causes \ref{eq:mip-1} to imply $t_i\leq -z_i$. When $z_i=0$, either case can hold and $t_i=0$. The linear programming relaxation lets $a_i$ be in the range $[0,1]$ instead of $\{0,1\}$. 

For any fixed $z_i$, we can compute the maximum value of $t_i$ under this relaxation, which is a function of the now-continuous variable, $a_i$. By setting the upper bounds to equality, the optimal value of $a_i$ is $a_i=\frac{z_i-l_i}{u_i-l_i}$ and $t$ is then upper bounded by 
\[t_i \leq \frac{-z_i + 2u_i\cdot (z_i-l_i)}{u_i-l_i}\]. We observe that this is an equivalent relaxation to the upper-hull provided by the absolute value operator and our vertical-parallelogram fitting procedure (next section). This allows us to rewrite the optimization above as 
\[\max_{z \in Z(c,E)} \sum_{i\in S} \frac{-z_i + 2u_i\cdot (z_i-l_i)}{u_i-l_i} - \sum_{i\in S^-} z_i + \sum_{i\in S^+} z_i
\]
which we notice is a linear program over a zonotope. The objective vector may be developed in $O(d)$ time, and linear programs may be solvable over zonotopes in $O(|E|)$ time. 
\end{proof}

\newpage
\section{Pseudocode}
\begin{algorithm}
  \caption{ZLip}
  \begin{algorithmic}[0]
    \Require{$L$-layer feedforward neural network $f$, input set $\mathcal{X}$, norm $\beta$}
    \Ensure{Zonotope $\mathcal{Y}_0 \supseteq \{\nabla_xf(x)^Tv \mid x\in\mathcal{X},\quad v\in B_{\beta^*}\}$}
    \Statex
    \Function{ZLip}{$f, \mathcal{X}, \beta$}
      \Let{$\mathcal{Z}_0$}{\texttt{Zonotope}$(\mathcal{X})$} \Comment{Cast input set to zonotope} 
      \For{$i \gets 1 \textrm{ to } L$} \Comment{Forward pass (e.g., DeepZ)}
        \Let{$\hat{\mathcal{Z}}_i$}{\texttt{map\_affine}$(W_i, b_i, \mathcal{Z}{i-1})$} 
        \Let{$\mathcal{Z}_i$}{\texttt{map\_nonlin}$(\sigma, \hat{\mathcal{Z}}_i)$}
        \Let{$\mathcal{J}_i$}{\texttt{elementwise\_jacobian}$(\sigma, \mathcal{Z}_i)$} \Comment{Gradient range for $\nabla_z\sigma(\mathcal{Z}_i)$}
      \EndFor 
    \Let{$\hat{\mathcal{Y}}_L$}{\texttt{Zonotope}$(\mathcal{B}_{\beta^*})$}  \Comment{Cast dual ball to zonotope}
    \For{$i \gets L \textrm{ to } 1$}\Comment{Backward pass}
        \Let{$\mathcal{Y}_{i-1}$}{\texttt{map\_affine}$(W_i^T, 0, \hat{\mathcal{Y}}_i)$}
        \Let{$\hat{\mathcal{Y}}_{i-1}$}{\texttt{elementwise\_mul}$(\mathcal{J}_i$, $\mathcal{Y}_{i-1})$}
        \EndFor 
    \Let{$\hat{\mathcal{Y}}_0$}{\texttt{map\_affine}($W_0^T, 0, \hat{\mathcal{Y}}_0)$}
    
    \State \Return{$\hat{\mathcal{Y}}_0$} 
    \EndFunction
  \end{algorithmic}
\end{algorithm}

\begin{algorithm}
\caption{Vertical Parallelogram Fitting}
\begin{algorithmic}[0]
    \Require{Function $\sigma:\mathbb{R}\to\mathcal{P}(\mathbb{R})$, and interval $[c-|E|, c+|E|]$}
    \Ensure{Slope $\Lambda^*$, Altitude $\mu^*$, center $b^*$}
  \Function{\texttt{VP\_Fit}}{$\sigma, c, E$}
  \Let{$\mathcal{I}$}{$c\pm|E|$}
  \Let{S}{$\{(x, \sigma(x))\; \mid \; x\in\mathcal{I}\}$} 
  \Let{$h^-, h^+$}{\texttt{conv\_hull}$(S)$} \Comment{Possibly hard, depends on $\sigma$}
  \Let{$x^*$}{arg$\max_{x\in \mathcal{I}} h^+(x)-h^-(x)$}
  \Let{$\mu^*$}{$h^+(x^*)-h^-(x^*)$} \Comment{Altitude} 
  \Let{$\Lambda^*$}{$\delta(h^-(x^*)) \cap \delta(-h^+(x^*))$} \Comment{Slope of parallogram's non-vert side}
  \Let{$b^*$}{$\frac{1}{2} \cdot (h^+(x^*)+h^-(x^*))$} \Comment{Intercept}
  \Statex \Return{$\Lambda^*, \mu^*, b^*$}
  \EndFunction
\end{algorithmic}
\end{algorithm}

\begin{algorithm}
  \caption{Auxiliary Functions}
  \begin{algorithmic}[0]
    \Function{\texttt{map\_nonlin}}{$\sigma, Z(c,E)$} \Comment{$Z(c,E)\subseteq \mathbb{R}^d$, $\sigma$ is elementwise}
    \For{$i \gets 1 \textrm{ to } d$}
    \Let{$\Lambda_i, b^*_i, \mu^*_i$ }{\texttt{VP\_Fit}$(\sigma_i, c_i, E_i^T)$}
    \EndFor
    \State \Return{$Z(b^*, \texttt{diag}(\mu^*)) \oplus(\Lambda \odot Z(c,E))$}
    \EndFunction
    \Statex
  \Function{\texttt{elementwise\_mul}}{$H(l,u), Z(c,E)$} 
  \For{$i\gets 1\textrm{ to } d$}
    \Let{$\Lambda_i, b^*_i, \mu^*_i$}{\texttt{VP\_Fit}$([l_i, u_i], c_i, E_i^T)$} \Comment{Overloading \texttt{VP\_Fit} signature}
  \EndFor
  \State \Return{$Z(b^*, \texttt{diag}(\mu^*)) \oplus(\Lambda \odot Z(c,E)) $}
  \EndFunction
  \end{algorithmic}
\end{algorithm}

\newpage

\section{Detailed Derivations for Vertical Parallelogram Fitting} 
We recall the algorithm from Section \ref{sec-vpfit} for fitting a vertical parallelogram to a 2-dimensional set $S$. The first step was to compute the upper and lower convex hulls of $S$. For sets of the form $\{(x, f(x)) \mid x\in [l, u]\}$ for some differentiable function, this is equivalent to the biconjugate and the biconjugate of the negation of $f$. In this case, there exists a simple algorithm to yield the upper-convex hull. The lower-convex hull may be found the same way, but for the set $\{(x, -f(x))\}$. Observe that if $f$ is convex over $[l,u]$, then the upper convex hull is the secant line between the endpoints $(l, f(l))$ and $(u, f(u))$ and the lower-convex hull is $f(x)$; vice versa for concave functions. For functions that are neither convex or concave over $[l,u]$, the upper convex hull may be piecewise continuous, alternating between secant-line segments and $f$. For function $f$, let the secant line of $f$ between $x_1$ and $x_2$ be denoted as $Sec^f(x_1,x_2)$.

We take inspiration from the gift-wrapping procedure for finding convex hulls of finite 2-dimensional point sets. In gift-wrapping, the idea is to find the left-most point in the set and sweep a ray clockwise until an intersection with another point in the set is found. The sweep continues, with ray now starting at the newly intersected point until the left-most point is intersected again. Our procedure sweeps performs a sweep over rays of decreasing slope, noting that any intersecting point must lie on the set $S$ and thus the ray is a secant line. Hence, the key subroutine to find the upper hull is to solve a maximization over slopes of secant lines. For a fixed $x_0$, the slope of the secant line $Sec^f(x_0, x)$ is $\frac{f(x)-f(x_0)}{x-x_0}$, and the maximization we seek to solve is a constrained variant of this,
\begin{equation}\label{eq:sec-slope}
\max\limits_{x\in [x_0, u]} \frac{f(x)-f(x_0)}{x-x_0}.
\end{equation}
When $f$ is differentiable, then we can differentiate the above objective and set to zero and solve for $x$ in the equality 
\begin{equation}\label{eq:giftwrap}
    f'(x) = \frac{f(x) - f(x_i)}{x-x_i}
\end{equation}
This procedure may be repeated until the max is attained at an $x \geq u_i$, for which the final secant line spans between $(x_i, f(x_i))$ and $(u, f(u))$. 

Once piecewise forms for the convex upper and lower hulls are formed, the maximal altitude can be computed by maximizing the piecewise function $h^+(x)-h^-(x)$. The proper slope for the tightest fitting vertical parallelogram is attained by considering an element in the intersection of subgradients of $h^-$ and $-h^+$ at their maximal altitude.

\subsection{Sigmoid/Tanh} 

\begin{figure*}[ht]
    \centering
    \includegraphics[width=0.32\textwidth]{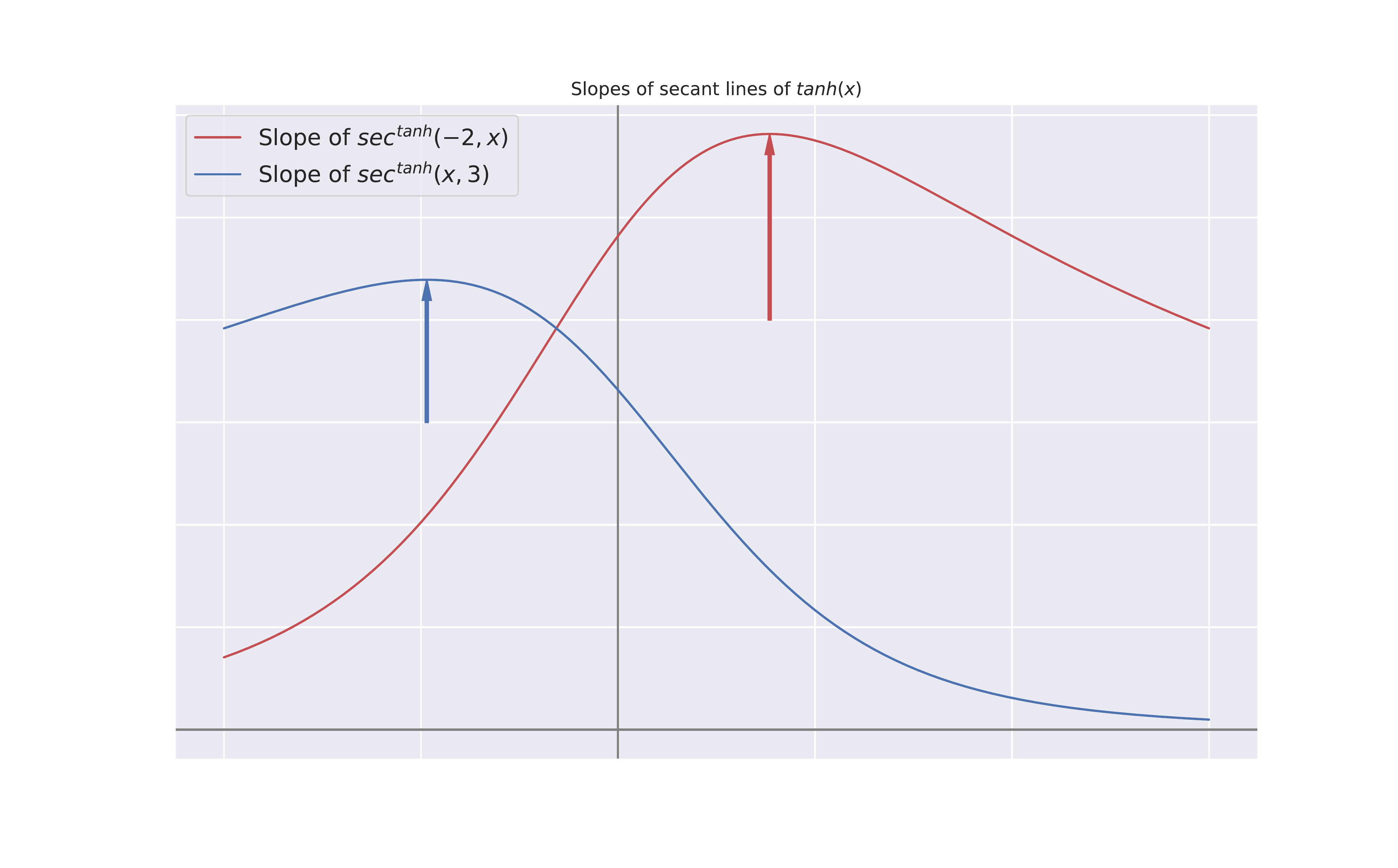}
    \includegraphics[width=0.32\textwidth]{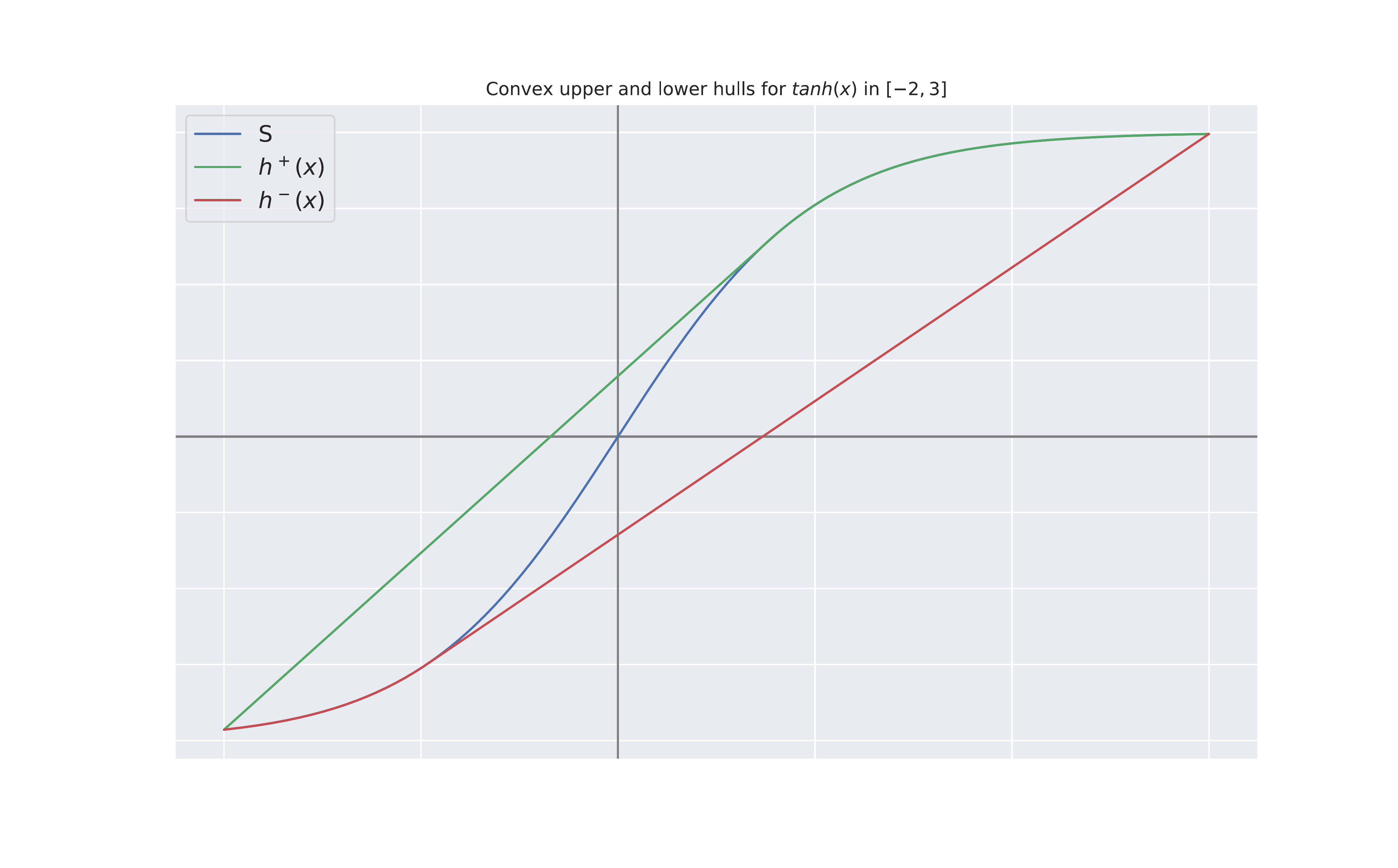}
    \includegraphics[width=0.32\textwidth]{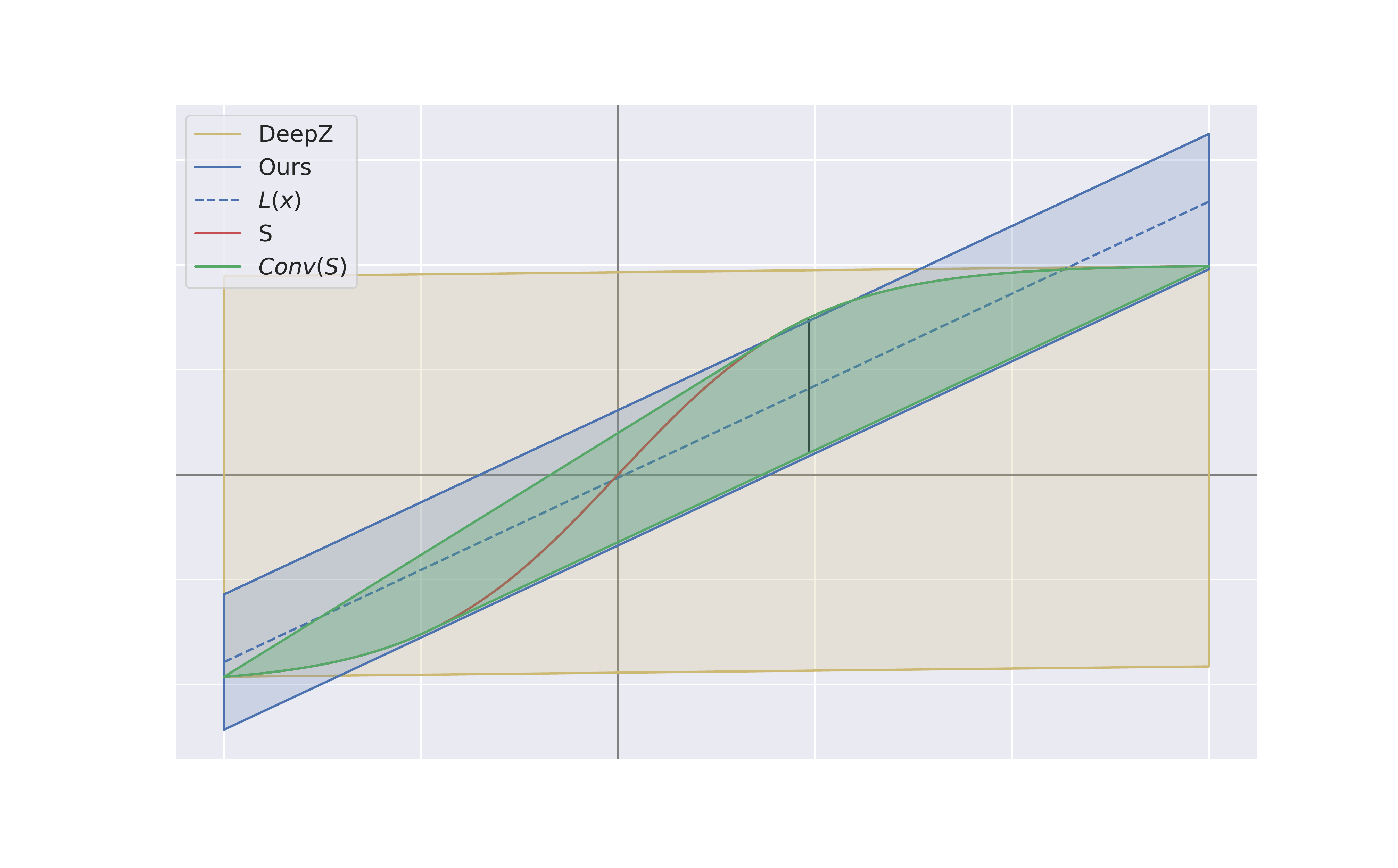}\\
    \vspace{-1em}
    \caption{Procedure for computing the tightest fitting vertical-parallelogram to the set $S=\{(x, tanh(x)) \mid x\in [-2, 3]\}$. (Left) We plot the slope of the secant-line from $-2$ to $x$ in red and the secant line from $x$ to 3 in blue, where the maximum is marked with a vertical hash. (Middle) We plot the upper and lower convex hulls for $S$. (Right) We plot the convex hull of $S$ and the parallelogram we produce versus that of prior work.}
\end{figure*} 
We demonstrate the above procedure for finding convex upper and lower hulls of sets $\{(x, f(x)\mid x\in [l,u]\}$ where $f$ is $S$-shaped like sigmoid and tanh. We say a function $f$ is $S$-shaped if it is monotonically increasing and there exists an $x'$ such that $f$ is concave for all $x\geq x'$. We break this into cases, based on the values of $[l,u]$. If $f$ is either convex or concave over the entire interval $[l,u]$, then the upper hull is either the secant line or $f$ respectively, and vice versa for the lower hull. In this case, the maximum altitude is attained at the $x$ where $f'(x)$ is equal to the slope of this secant line.

For cases where $f$ is both convex and concave for portions of $[l,u]$, then we proceed with the iterative giftwrap procedure, starting at $x_0=l$. For monotonically increasing tanh and sigmoid functions, this means that $l < 0 < u$ and we seek to find a point in the concave ($x \geq 0$) portion of $f$ such that Equation \ref{eq:giftwrap} is satisfied. Since $f$ is concave for $x\geq 0$, the function 
\[ f'(x) - \frac{f(x)-f(l)}{x-l}\] 
only has one zero for $x \geq 0$, which may be found numerically\footnote{We keep track of the numerical error and ensure that the vertical height of the parallelogram accounts for this.}. Letting $x^*$ be such an $x$, we have that the secant-line is given by $y=\frac{f(x^*)-f(l)}{x^*-l}(x-l) + f(l)$ and the upper convex hull for is 
\[
h^+(x)=\begin{cases}
         \frac{f(x^*)-f(l)}{x^*-l}(x-l) + f(l) \quad &\text{if} \, x \in [l, x^*] \\
          f(x) \quad &\text{if} \, x \in [x^*, u] \\
     \end{cases}
\]
A similar procedure may be considered for the lower convex hull, also yielding a convex hull like 
\[
h^-(x)=\begin{cases}
        f(x) \quad &\text{if} \, x \in [l, x^\dagger] \\
         \frac{f(x^\dagger)-f(l)}{x^\dagger-l}(x-l) + f(l) \quad &\text{if} \, x \in [x^\dagger,u ] \\
          
     \end{cases}
\]
where $x^\dagger$ is the minimum amongst zeros for the function $f'(x) - \frac{f(x)-f(u)}{x-u}$. Then the altitude, $h^+(x)-h^-(x)$, is a concave piecewise function with three pieces, segmented into the intervals $[l, x^\dagger], [x^\dagger, x^*], [x^*, u]$. If the slope of the linear component of the upper hull is $\lambda^+$ and the slope of the linear component of the lower hull is $\lambda^-$, then the maximum is attained at one of four points: i) the point in $[l, x^\dagger]$ where $f'(x)=\lambda^-$; ii) the point in $[x^*, u]$ where $f'(x)=\lambda^+$; iii) $x^*$, or iv) $x^\dagger$. It is trivial to check all these points where $f$ is the sigmoid or tanh function. 

Regardless of which case we are in, once the upper and lower hulls have been computed and their difference has been maximized, the remaining steps are simple. It suffices to compute the subgradients of the lower hull and negative upper hull at that point, pick an element of their intersection and find the line that passes through the proper midpoint. This yields the right altitude and affine function for a vertical parallelogram. 

\subsection{Elementwise Multiplication } 
Now we consider the full set of cases for elementwise multiplication. Letting $S$ be the 2-dimensional set $\{(x, y\cdot x) \mid x\in [l, u],\quad y\in[\alpha,\beta]\}$. When $\alpha=\beta$, then the set $S$ is a line segment, and has vertical parallelogram with height $0$ and slope $\alpha$. When $\alpha < \beta$, and $l > 0$ the convex upper hull is $h^+(x)=\beta x$ and the convex lower hull is $h^-(x)=\alpha x$; vice versa for when $u < 0$. The only remaining case is when $\alpha < \beta$, and $l < 0< u$, which is the case considered in the main paper, where $h^+(x)$ is the line passing through $(l, \alpha l)$ and $(u, \beta u)$; and the lower hull passes through $(l, \beta l)$ and $(u, \alpha u)$. 

In all of the above cases, the upper and lower hulls are affine functions and their maximum over $[l,u]$ is attained at either $l$ or $u$. The admissable slopes are exactly the range $[\alpha, \beta]$.

\subsection{Absolute Value}
\begin{wrapfigure}{r}{0.32\textwidth}
 \includegraphics[width=0.32\textwidth]{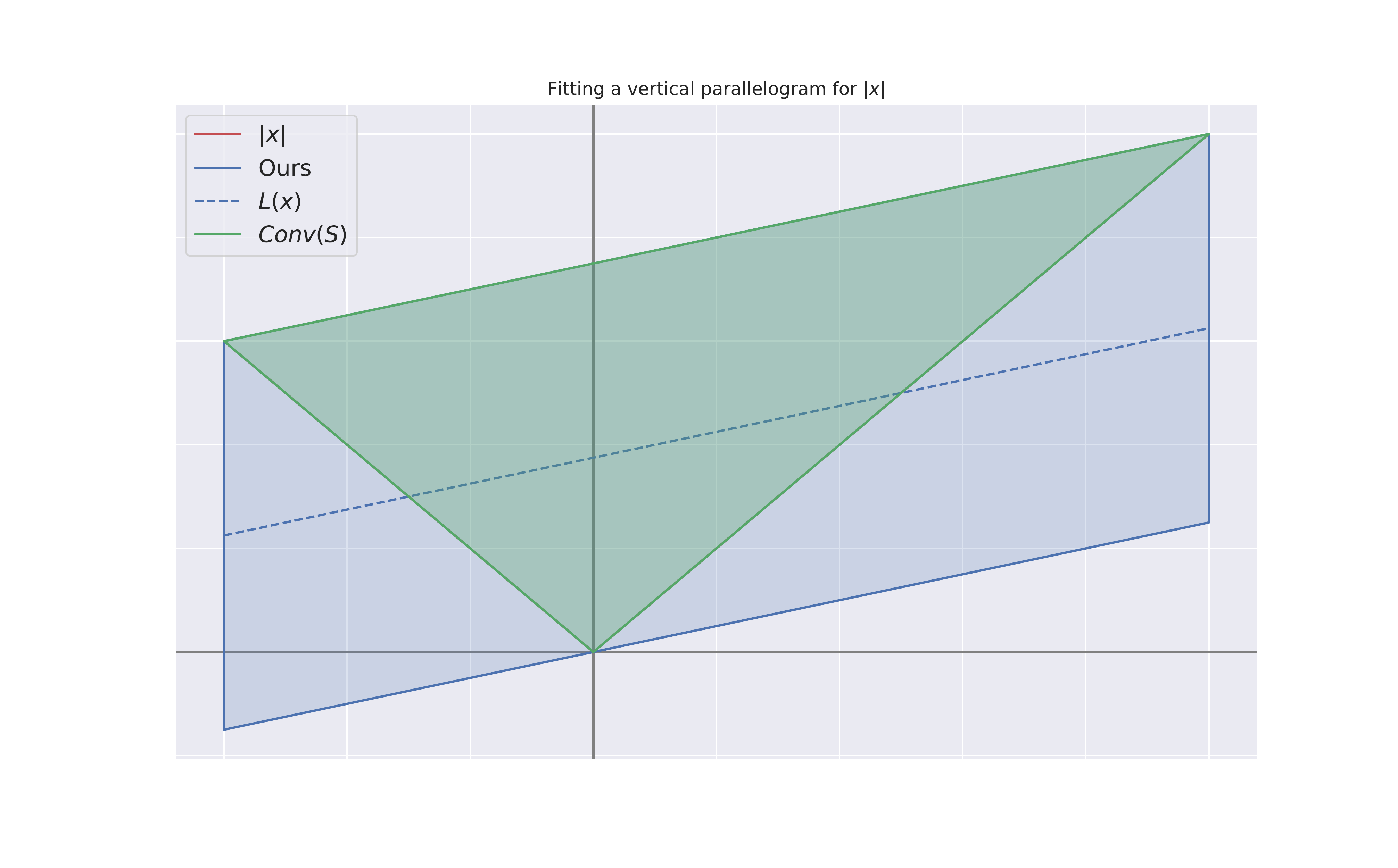}
 \caption{Convex hull and vertical parallelogram for the set $S=\{(x, |x|)\mid x\in[-3, 5]\}.$}
\end{wrapfigure} 
Now we consider the absolute value function, as used in the proof of Theorem \ref{thm:lp-relax}. Consider the set $\{(x, |x|)\mid x\in [l,u]\}$. When $l\cdot u\geq 0$, then the set $S$ lies on either the $y=x$ or $y=-x$ line, and the vertical parallelogram fit is equivalent to mapping through an affine function. On the other hand, when $l < 0 < u$, the set $S$ has a lower convex hull of $h^-(x)=|x|$, because $|x|$ is a convex function. The upper hull is the secant line connecting $(l, |l|)$ and $(u, |u|)$, which may be written as $h^+(x)=\frac{u+l}{u-l}(x) - \frac{2ul}{u-l}$. The maximum of the altitude, $h^+(x)-h^-(x)$, is attained at $x=0$, for which the maximum altitude is $-\frac{2ul}{u-l}$ and the height of the vertical parallelogram is $\mu=\frac{-ul}{u-l}$. The slope must be the slope of the (linear) upper hull, $\frac{u+l}{u-l}$ and the central line of the parallelogram must pass through the point $(0, \frac{-ul}{u-l})$. 
\newpage
\section{Experiments}
\subsection{Model Architectures} 
Here we will describe the structure of each of the architectures considered. For networks with only fully connected and elementwise nonlinearities, we denote the architectures by $[n_{in}, n_1,\dots n_L\dots n_{out}]$, where $n_i$ denotes the number of neurons in the $i^{th}$ hidden layer, and ReLU nonlinearities are implied between each layer. We will use the notation ``FC X'' to denote fully connected layers with an output of X neurons, and $\text{Conv}_s (C\times W \times H)$ to denote convolutional layers with a stride of $s$ and output dimension of $C$ channels, and kernels of size $W\times H$. Transpose convolutional layers are denoted $\text{Conv}_s^T (C\times W \times H)$. For layers of the same size repeated $k$ times, we'll denote this as $[layers]^{\times k}$. 
\paragraph{Toy Networks:} For the networks trained on the toy dataset, the input and output dimension are each 2, and the scalar-valued output is attained by taking the dot product with the vector $[+1, -1]$. These have varying depth, but all have architectures like 
\[x \to \Big[FC 100 \to ReLU\Big]^{\times k} \to FC 2 \to z \] 
where the number of hidden layers denotes the number of ReLU layers in the network.

\paragraph{Generators for MNIST and CIFAR:} For MNIST and CIFAR, we trained 6 VAEs and 2 GANs. These each have the same architecture with the exception of varying input/output shapes. The VAEs each have a latent dimension of 20 and the GANs have an input dimension of 100. We use the notation $D$ and $C$ to denote the output dimnension and channel: $(784, 1)$ for MNIST and $(3072, 3)$ for CIFAR-10.
\begin{itemize}
    \item \textbf{VAESmall:} $[D, 400, 200, 20, 200, 400, D].$ Where the decoder is just the $[20, 200, 400, D]$ subnetwork. 
    \item \textbf{VAEMed:} $[D, 400, 200, 100, 50, 100, 200, 400, D]$ where the decoder is just the $[50, 100, 200, 400, D]$ subnetwork. 
    \item \textbf{VAEBig:} $[D, 400, 200, 200, 200, 200, 100, 200, 200, 200, 200, 400, D]$, where the decoder is just the $[100, 200, 200, 200, 200, 400, D]$ subnetwork. 
    \item \textbf{VAECNN:} $x \to \text{Conv}_2 (16\times 4\times 4) \to ReLU \to \text{Conv}_2 (32 \times 4 \times 4) \to ReLU \to FC 50 \to FC 800 \to ReLU \to \text{Conv}_2^T (16 \times 5 \times 5) \to ReLU \to \text{Conv}_2^T (C\times 4 \times 4) \to Sigmoid.$
    \item \textbf{VAETanh}: Same as VAEMed with $tanh$ nonlinearities in place of ReLU. 
    \item \textbf{VC-Tanh}: Same as VAECNN with $tanh$ nonlinearities in place of ReLU. 
    \item \textbf{FFGAN}: $[100, 256, 512, 1024, D].$
    \item \textbf{DCGAN}: $x \to \text{Conv}_1^T (256 \times 4 \times 4) \to ReLU \to \text{Conv}_2^T (128 \times 4 \times 4) \to ReLU \to \text{Conv}_2^T (64\times 4 \times 4) \to ReLU \to \text{Conv}_2^T (C \times 4 \times 4) \to tanh.$
\end{itemize}

\paragraph{Classifiers for MNIST and CIFAR:} For MNIST and CIFAR, we trained three fully connected networks. We refer to these as $\{tiny, small, med,\}-*$. Each of these were trained with both the ReLU and tanh nonlinearities with both standard and PGD adversarial training. We will describe the training techniques in the next section. 
\begin{itemize} 
\item \textbf{Tiny*}: $[D, 20, 20, 10].$
\item \textbf{Small*}: $[D, 100, 100, 100, 10].$
\item \textbf{Med*}: $[D] +[100]\times 6 + [10].$
\end{itemize}
\subsection{Datasets, Training Methods, and Computing Environment} 
\paragraph{Computing Environment: } All networks were trained using Pytorch on a machine with $2\times$ GeForce RTX 2070 GPU's. All Lipschitz evaluations were performed using the CPU only, an Intel i7-9700K. Mixed integer programming evalutaions were performed using 4 cores and the Gurobi optimizer.
\paragraph{Datasets:} 
The CIFAR-10 and MNIST datasets are standard. For the dataset taken from \cite{Aziznejad2020-aq}, we generate 1000 points uniformly randomly in $[-1,1]^2$ and attach the label 1 if the norm of the data point is $\leq \frac{2}{\pi}$.
\paragraph{Training Methods:} 
We outline the methods used to train each set of networks: 
\begin{itemize}
    \item \textbf{Toy Dataset Nets:} We trained using the CrossEntropy loss and the Adam optimizer with a learning rate of $0.001$, where we trained for 200 epochs.
    \item \textbf{VAEs:} For the VAE networks, we train with the standard VAE loss: a sum of binary cross-entropy between the reconstructed example and original input, and a KL-divergence term. We train for 50 epochs for all VAEs with the Adam optimizer and a learning rate of 0.001
    \item \textbf{GANs:} For the GANs, we train using the binary cross-entropy applied to the discriminator output. We train with Adam $(lr=0.001)$ and 25 epochs. 
    \item \textbf{Classifiers:} For the classifiers considered, we always train for 50 epochs using Adam $(lr=0.001)$ and either the standard Cross-Entropy loss or the PGD loss for adversarial training. For MNIST networks, we allow an adversarial budget of $\ell_\infty$ norm of 0.1, and the adversary takes 10 steps with step-size $0.2$. For CIFAR-10, the adversary has a budget of $\frac{8}{255}$ and takes 10 steps of stepsize $\frac{2}{255}.$
\end{itemize}

\subsection{Varying Radius} 

\begin{wrapfigure}{r}{0.32\textwidth}
 \includegraphics[width=0.32\textwidth]{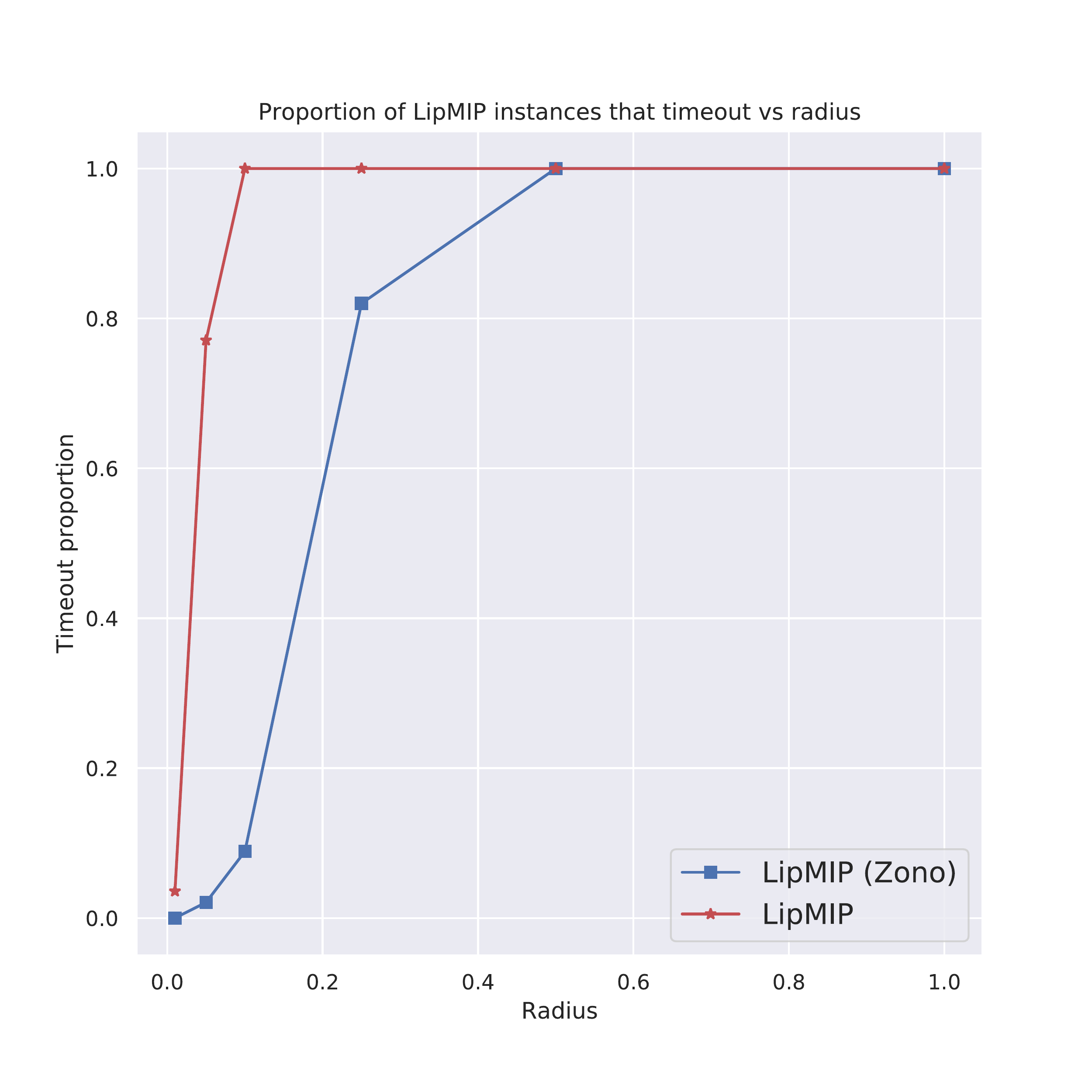}
 \caption{Proportion of timed out examples for LipMIP versus LipMIP using ZLip as a first step across varius radii. }\label{fig:lipmip-timeout}
\end{wrapfigure} 
We consider the effect of varying the radius of the region we evaluate Lipschitz constants over. We expect that as the radius increases, the bounds for the preactivations at each layer will become more loose. For ReLU networks, when zero is strictly contained in the preactivation bounds, a new degree of freedom needs to be introduced to the layerwise zonotope approximations. Hence, as more degrees of freedom are introduced, we expect the bound returned by ZLip to be looser and the runtime will be less efficient since the representation of each zonotope will be larger. However, other methods will likely also yield looser bounds as well. We also note that a more intelligent strategy that prunes the zonotope representation size to a more managable size may be employed, though we leave such performance improvements for future work. 

\paragraph{Toy Dataset:} First we examine the effect of changing the evaluation radius on the network trained on the circle dataset with 6 hidden layers of width 100 and the ReLU nonlinearity. We report the results in table \ref{table:toy-comparison}, where we have evaluated over 64 elements from the test set  and $\ell_\infty$ balls of each radius. We set a timeout of 120s for the mixed-integer programming approaches, at which point the tightest upper bound is returned, explaining the discrepancy between the results on the first two columns. The proportion of timed out examples is presented in Figure \ref{fig:lipmip-timeout}. We observe that for networks with small input dimension, CLEVER is fairly accurate and can be viewed as a surrogate for the true Lipschitz constant when the LipMIP results time out. In this case we see that for all but the global evaluation returned by SeqLip, all techniques provide looser bounds as the radius increases, however ZLip remains significantly tighter than Fast-Lip and even yields a tighter result than the tightest upper bound provided  by LipMIP after this procedure times out. 

\begin{table}[]
\centering
\caption{Evaluation of various Lipschitz computation techniques for the network with 6 hidden layers trained on the Circle dataset. }\label{table:toy-comparison}
\begin{tabular}{lllllll}
\hline 
\multicolumn{7}{c}{Lipschitz Values for $6\times100$ Circle Network}                                    \\ \hline
\multicolumn{1}{l|}{Radius} & LipMIP (Zono) & LipMIP    & ZLip      & Fast-Lip  & SeqLip    & CLEVER    \\ \hline
\multicolumn{1}{l|}{0.01}   & 2.62\e{2}     & 2.97\e{2} & 2.98\e{2} & 2.82\e{3} & 2.72\e{3} & 2.95\e{2} \\
\multicolumn{1}{l|}{0.05}   & 3.69\e{2}     & 2.79\e{3} & 4.22\e{2} & 9.94\e{3} & 2.72\e{3} & 3.61\e{2} \\
\multicolumn{1}{l|}{0.1}    & 4.55\e{2}     & 6.09\e{3} & 6.17\e{2} & 1.29\e{4} & 2.72\e{3} & 4.39\e{2} \\
\multicolumn{1}{l|}{0.25}   & 8.91\e{2}     & 1.67\e{4} & 1.55\e{3} & 2.00\e{4} & 2.72\e{3} & 4.65\e{2} \\
\multicolumn{1}{l|}{0.5}    & 2.51\e{3}     & 2.20\e{4} & 3.70\e{3} & 2.46\e{4} & 2.72\e{3} & 6.72\e{2} \\
\multicolumn{1}{l|}{1.0}    & 6.49\e{3}     & 2.82\e{4} & 5.81\e{3} & 3.10\e{4} & 2.72\e{3} & 9.04\e{2}
\end{tabular}

\end{table}

\paragraph{Generative Models:} 
Here we present results on the generative models when we vary the radius of the input region. We focus primarily on the VAEs described above for both the MNIST and CIFAR-10 datasets. These results for VAEMEd on MNIST are displayed in Figure \ref{fig:mnist-radseries}. As expected, increasing radius size increases the estimated Lipschitz bound for both ZLip and Fast-Lip.  The relative gap decreases as the radius increases, indicating that ZLip is comparatively tighter when few neurons are unstable, which we attribute to ZLip being able to perfectly map the affine layers. It is also expected that Fast-Lip does not get slower even as more neurons become unstable, as the representation size of a hyperbox in $\R^d$ will always be $O(d)$, and looser neuron bounds just increases the looseness of the approximation without increasing runtime. 

\begin{figure*}[ht]
    \centering
    \includegraphics[width=0.4\textwidth]{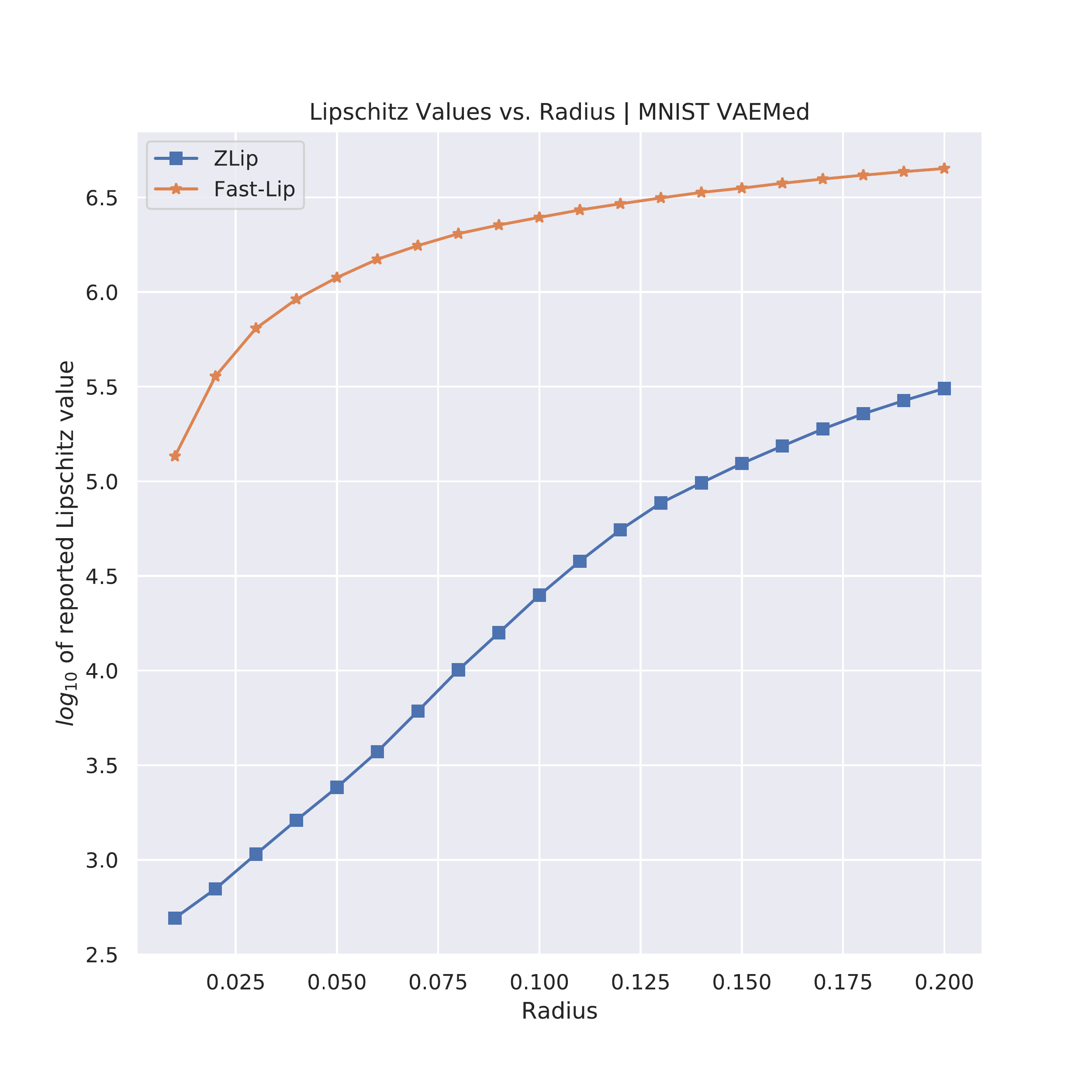}
    \includegraphics[width=0.4\textwidth]{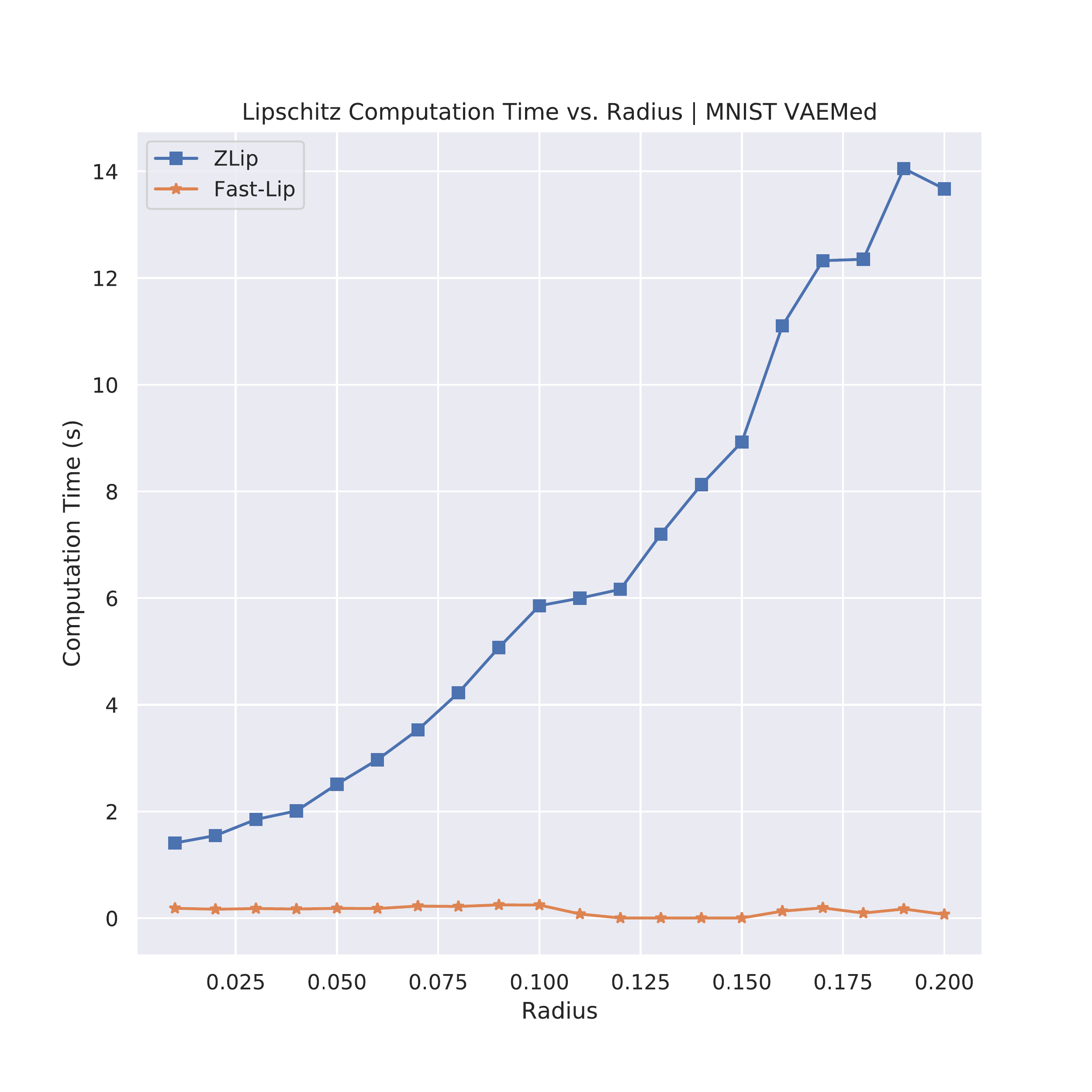}\\
    \includegraphics[width=0.4\textwidth]{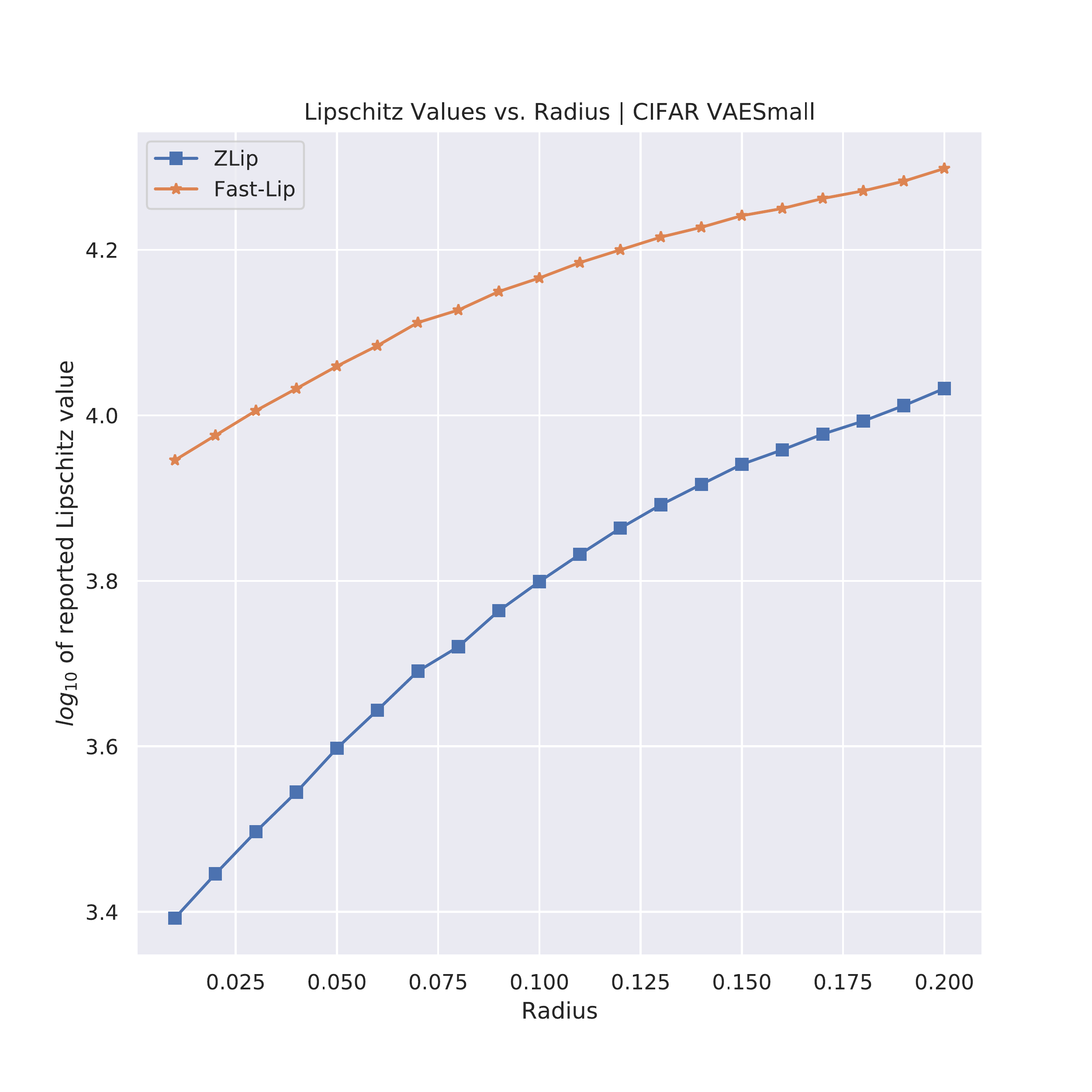}
    \includegraphics[width=0.4\textwidth]{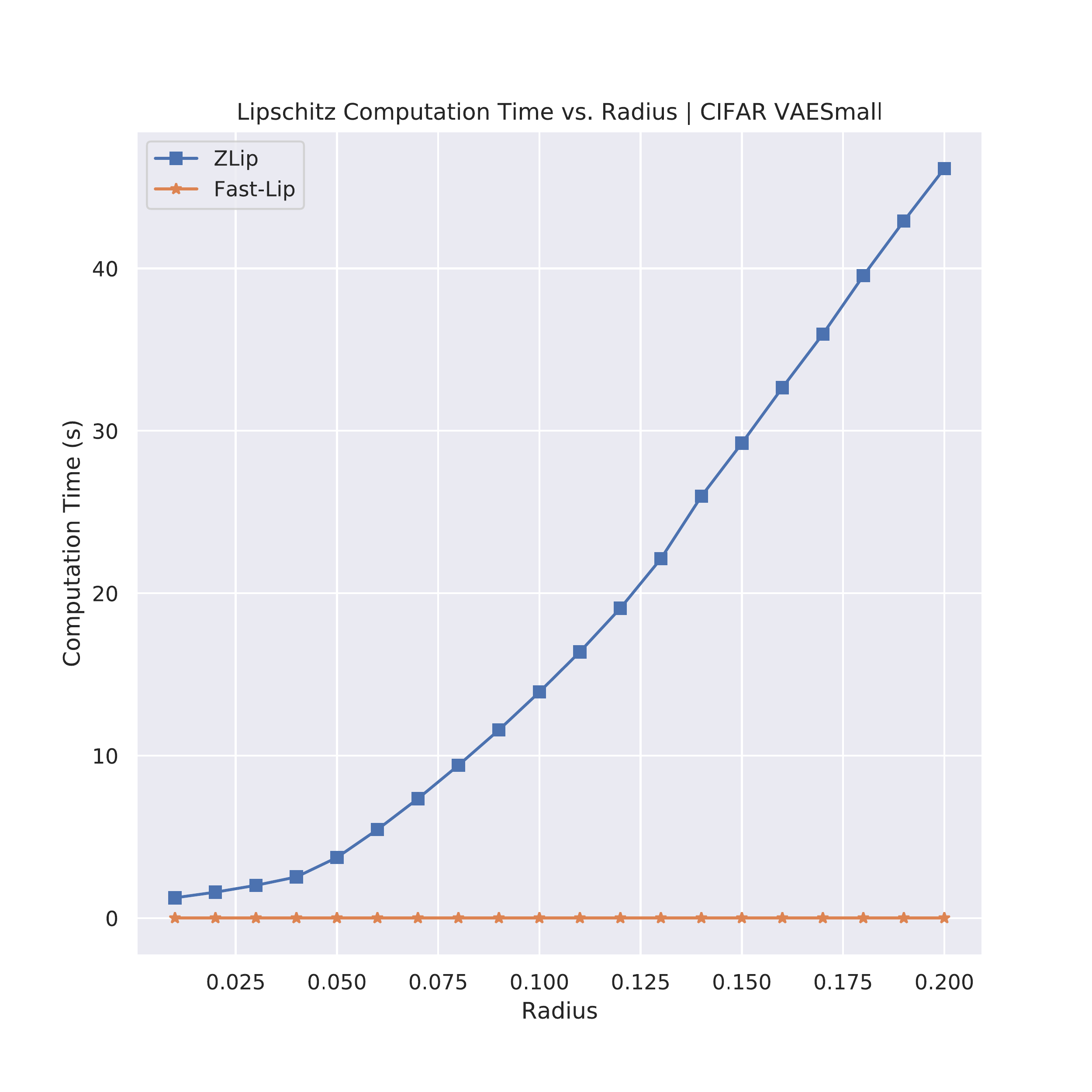}
    \vspace{-1em}
    \caption{Reported Lipschitz constants and times for the MNIST MedVAE as we increase the radius of the region over which we evaluate Lipschitz constant. (Top Left) reports the values on a log-scale, noticing that both ZLip and Fast-Lip increase their estimate as the radius increases, the relative gap is largest for small radii.(Top Right) we display times in seconds versus the radius. The increase in running time of ZLip is due to the increase in number of unstable neurons, which increases the size of the representations of the zonotopes that must be passed through each layer. (Bottom) Reported Lipschitz constants and times for the CIFAR SmallVAE as we increase radius.}
   \label{fig:mnist-radseries}
\end{figure*} 

\FloatBarrier
\subsection{Ablating the choice of abstract domain:}
ZLip operates by iteratively building \emph{zonotopes} to satisfy the containments of equations \ref{eq:setrec-inputs}-\ref{eq:setrec-jac}, specifically we use zonotopes in both the forward pass (such as DeepZ, DiffAI), but also zonotopes in the backward pass. To examine the importance of zonotopes in both directions, we replace the zonotopes with hyperboxes in one or both of the directions. Note that using hyperboxes in both the forward and backward directions is FastLip. We denote the method that uses hyperboxes in the forward pass, but zonotopes in the backward pass as `Hyperbox -> Zono', and vice versa. 

\paragraph{Toy Dataset:} In Figure \ref{fig:circle-ablation} we compare the performance of the different techniques on networks trained on the toy dataset. The y-axis records the logarithm of the average ratio between each method's Lipschitz estimate and ZLip's Lipschitz estimate for each example. We evaluate on 100 random points for each x-value. On the left panel we vary the architecture size while evaluating on inputs that are hyperboxes with radius 0.1. For every network considered, on average, the abstract domain in the forward pass is more important, and this gap becomes more apparent as the size of the network increases. On the right panel, we fix a network and evaluate the performance as we vary the size of the certified region. For small radii, the choice of the forward domain is more important, however as the radius increases, the backward domain becomes more important. We conjecture this is because, for large radii, there is much uncertainty about which ReLU's are fixed and the performance of zonotopes and hyperboxes in the forward pass becomes equivalent. 

\begin{figure*}[ht]
    \centering
    \includegraphics[width=0.48\textwidth]{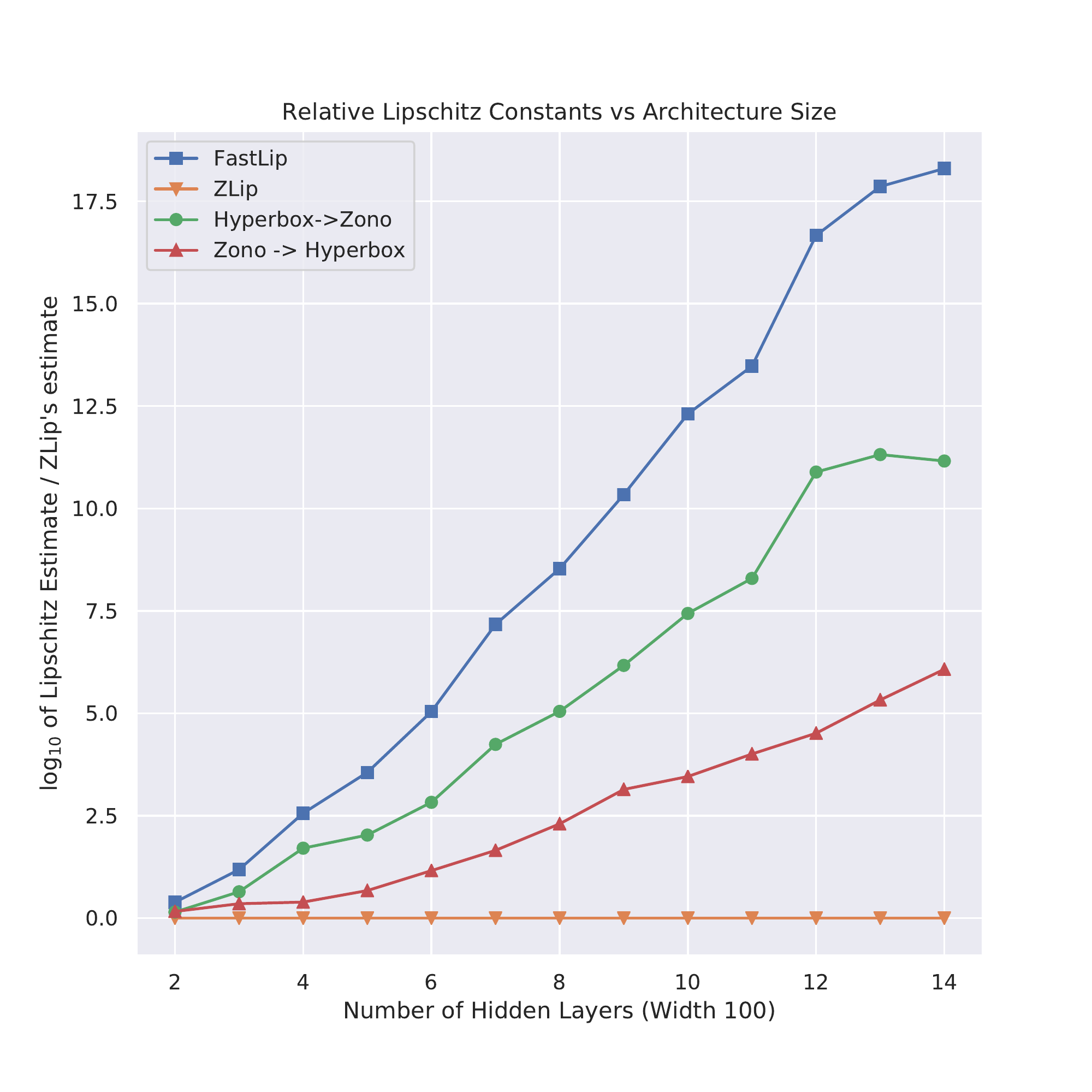}
    \includegraphics[width=0.48\textwidth]{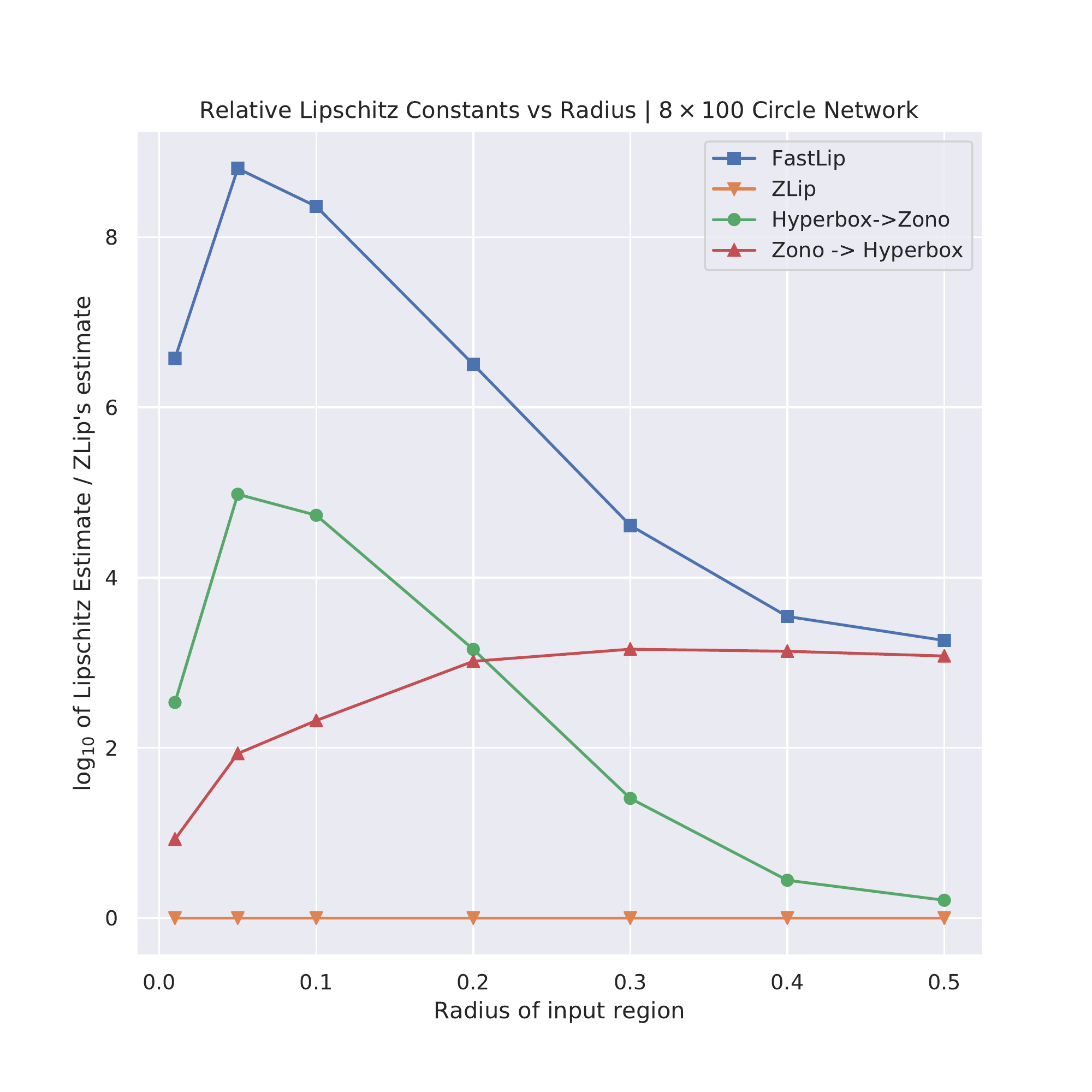}
    \vspace{-1em}
    \caption{Log-mean-ratio of reported Lipschitz estimates (relative to ZLip) on varying networks trained on the Circle dataset. (Left) reports the values as we vary the network size, demonstrating that the forward domain is more important as the network size increases. (Right) reports the values for a fixed network as we vary the input radius size. Larger input radii have more uncertain ReLU neurons and the choice in forward domain becomes relatively less important.}
   \label{fig:circle-ablation}
\end{figure*}

\paragraph{MNIST Generative models:} 
In Table \ref{table:vae-ablation} we evaluate the choice of abstract domains on MNIST generative models. We consider the decoders of the three different VAE's trained on MNIST and evaluate over two different radii, considering hyperboxes of the denoted radius surrounding the encodings of random test MNIST examples. Here we again see that neither of the single-zonotope approaches is unilaterally better, however larger models tend to benefit more from using zonotopes in the forward pass. 
\begin{table}[]
\centering
\caption{Mean ratio of Lipschitz estimates provided by other abstract domains relative to the estimate provided by ZLip, evaluated on MNIST VAE decoders. We see that as the model becomes larger, using zonotopes in the forward pass yields a tighter bound. }\label{table:vae-ablation}
\begin{tabular}{@{}l|lll|lll@{}}
\toprule
Radius   & \multicolumn{3}{c|}{0.01}                                                                                  & \multicolumn{3}{c}{0.1}                                                                                  \\ \cmidrule(l){2-7} 
Method   & \multicolumn{1}{c}{FastLip} & \multicolumn{1}{c}{$H\rightarrow Z$} & \multicolumn{1}{c|}{$Z\rightarrow H$} & \multicolumn{1}{c}{FastLip} & \multicolumn{1}{c}{$H\rightarrow Z$} & \multicolumn{1}{c}{$Z\rightarrow H$} \\ \midrule
VAESmall & 7.85                        & \textbf{1.33}                        & 5.89                                  & 9.39                        & \textbf{3.08}                        & 3.11                                 \\
VAEMed   & 276.41                      & \textbf{12.71}                       & 43.78                                 & 410.92                      & 51.81                                & \textbf{12.25}                       \\
VAEBig   & 4238.83                     & 123.21                               & \textbf{85.65}                        & 26018.02                    & 1357.93                              & \textbf{26.64}                       \\ \bottomrule
\end{tabular}
\end{table}

\FloatBarrier
\subsection{Experimental Results on Classifiers} 
\paragraph{MNIST:}
For completeness and comparison against other networks on more realistic networks, we present results of our Lipschitz bounding technique versus several recent works for a variety of networks trained both with the standard classification loss as well as those trained adversarially. We evaluate the Lipschitz value returned by SeqLip, LipSDP, Fast-Lip, CLEVER, and ZLip for inputs of radius 0.1 centered at elements taken from the test set. If $f$ is the trained classifier which has outputs in $\R^{10}$, we consider the Lipschitz constant of the network $f_i(\cdot)-f_{i+1}(\cdot)$ for each example where the true label is $i$. 
First we present the values for the MNIST networks trained with the standard CrossEntropy Loss in Table \ref{table:mnist-val-xent} and times are presented in Table \ref{table:mnist-time-xent}. We remark that the bounds reported by CLEVER for networks with high dimension have been shown to be quite loose, so it is unclear what the correct Lipschitz value is for each of these networks. The most salient points here are that the values returned by ZLip are comparable to those returned by LipSDP at a significantly faster runtime. We also note that LipSDP errors when applied to networks as large as MedReLU or MedTanh. 

For the PGD trained MNIST networks, we present the values and times in tables \ref{table:mnist-val-pgd}, \ref{table:mnist-time-pgd}. In direct comparison to the tables for the networks trained with CrossEntropy loss, we notice that all methods report lower values for the adversarially trained network. This tracks with prior work that adversarial regularization serves as a form of Lipschitz regularization. 

\paragraph{CIFAR-10:}
The same experiments as above were performed on networks trained to classify the CIFAR-10 dataset. We present these results in Tables \ref{table:cifar-val-xent}-\ref{table:cifar-time-pgd}, but note that these results are qualitatively very similar to the results for the MNIST networks.

\begin{table}[]
\centering 
\caption{Lipschitz values reported by various networks evaluated on the various Classifiers described above. All numbers report an average over regions of radius 0.1 centered at examples from the test set. }\label{table:mnist-val-xent}
\begin{tabular}{@{}llllll@{}}
\toprule
\multicolumn{6}{c}{Lipschitz Estimates (MNIST) | CrossEntropy Loss}                        \\ \midrule
\multicolumn{1}{l|}{Network}   & SeqLip    & SDP       & Fast-Lip  & CLEVER    & ZLip      \\ \midrule
\multicolumn{1}{l|}{TinyReLU}  & 3.51\e{3} & 2.94\e{3} & 7.36\e{3} & 1.15\e{1} & 5.59\e{3} \\
\multicolumn{1}{l|}{SmallReLU} & 2.12\e{4} & 1.52\e{4} & 1.94\e{5} & 9.59\e{0} & 9.34\e{4} \\
\multicolumn{1}{l|}{MedReLU}   & 1.08\e{6} & ---       & 1.45\e{8} & 1.06\e{1} & 1.77\e{7} \\
\multicolumn{1}{l|}{TinyTanh}  & 1.33\e{4} & 1.14\e{4} & 2.56\e{4} & 2.97\e{1} & 1.98\e{4} \\
\multicolumn{1}{l|}{SmallTanh} & 5.80\e{4} & 4.29\e{4} & 3.24\e{5} & 3.26\e{1} & 1.68\e{5} \\
\multicolumn{1}{l|}{MedTanh}   & 5.50\e{6} & ---       & 4.29\e{8} & 3.56\e{2} & 6.94\e{7}
\end{tabular}
\end{table}

\begin{table}[]
\centering 
\caption{Times for MNIST classifiers trained with the Cross-Entropy loss}\label{table:mnist-time-xent}
\begin{tabular}{llllll}
\hline
\multicolumn{6}{c}{Lipschitz Times (MNIST) | CrossEntropy Loss}                               \\ \hline
\multicolumn{1}{l|}{Network}   & SeqLip     & SDP       & Fast-Lip   & CLEVER    & ZLip       \\ \hline
\multicolumn{1}{l|}{TinyReLU}  & 5.08\e{-1} & 1.78\e{1} & 1.75\e{-3} & 5.98\e{1} & 4.20\e{-1} \\
\multicolumn{1}{l|}{SmallReLU} & 2.28\e{0}  & 2.63\e{2} & 9.73\e{-2} & 1.01\e{2} & 9.36\e{-1} \\
\multicolumn{1}{l|}{MedReLU}   & 3.35\e{0}  & --- & 1.07\e{-1} & 1.75\e{2} & 2.13\e{0}  \\
\multicolumn{1}{l|}{TinyTanh}  & 5.90\e{1}  & 1.82\e{1} & 1.96\e{-3} & 5.98\e{1} & 4.22\e{-1} \\
\multicolumn{1}{l|}{SmallTanh} & 4.56\e{0}  & 2.28\e{2} & 1.05\e{-1} & 1.09\e{2} & 9.87\e{-1} \\
\multicolumn{1}{l|}{MedTanh}   & 9.91\e{0}  & --- & 1.08\e{-1} & 1.77\e{2} & 2.03\e{0} 
\end{tabular}
\end{table}

\begin{table}[]
\centering
\caption{Values for MNIST classifiers trained with the PGD loss}
\label{table:mnist-val-pgd}

\begin{tabular}{llllll}

\hline
\multicolumn{6}{c}{Lipschitz Estimates (MNIST) | PGD Loss}                                \\ \hline
\multicolumn{1}{l|}{Network}   & SeqLip    & SDP       & Fast-Lip  & CLEVER    & ZLip      \\ \hline
\multicolumn{1}{l|}{TinyReLU}  & 3.97\e{2} & 2.66\e{2} & 3.48\e{2} & 1.07\e{0} & 1.42\e{2} \\
\multicolumn{1}{l|}{SmallReLU} & 1.52\e{3} & 9.35\e{2} & 1.98\e{4} & 1.85\e{0} & 5.39\e{3} \\
\multicolumn{1}{l|}{MedReLU}   & 3.91\e{4} & 1.64\e{4} & 1.45\e{7} & 2.26\e{0} & 9.36\e{5} \\
\multicolumn{1}{l|}{TinyTanh}  & 1.38\e{3} & 8.61\e{2} & 6.86\e{2} & 3.14\e{0} & 5.17\e{2} \\
\multicolumn{1}{l|}{SmallTanh} & 7.96\e{3} & 4.81\e{3} & 5.68\e{4} & 6.24\e{0} & 2.42\e{4} \\
\multicolumn{1}{l|}{MedTanh}   & 2.73\e{5} & 1.30\e{5} & 5.98\e{7} & 1.90\e{1} & 6.27\e{6}
\end{tabular}
\end{table}

\begin{table}[]
\centering 
\caption{Times for MNIST classifiers trained with the PGD loss}
\label{table:mnist-time-pgd}

\begin{tabular}{@{}llllll@{}}
\toprule
\multicolumn{6}{c}{Lipschitz Times (MNIST) | PGD Loss}                                        \\ \midrule
\multicolumn{1}{l|}{Network}   & SeqLip     & SDP       & Fast-Lip   & CLEVER    & ZLip       \\ \midrule
\multicolumn{1}{l|}{TinyReLU}  & 1.18\e{-2} & 1.14\e{1} & 1.74\e{-3} & 2.71\e{1} & 1.10\e{-1} \\
\multicolumn{1}{l|}{SmallReLU} & 2.71\e{1}  & 1.77\e{2} & 1.13\e{-1} & 1.13\e{2} & 1.12\e{0}  \\
\multicolumn{1}{l|}{MedReLU}   & 3.10\e{0}  & 4.42\e{2} & 1.11\e{-1} & 1.76\e{2} & 2.17\e{0}  \\
\multicolumn{1}{l|}{TinyTanh}  & 6.11\e{-1} & 1.62\e{1} & 1.84\e{-3} & 5.33\e{1} & 3.77\e{-1} \\
\multicolumn{1}{l|}{SmallTanh} & 4.61\e{0}  & 1.77\e{2} & 8.66\e{-2} & 7.72\e{1} & 7.31\e{-1} \\
\multicolumn{1}{l|}{MedTanh}   & 3.29\e{0}  & 5.00\e{2} & 1.16\e{-1} & 1.86\e{2} & 2.15\e{0} 
\end{tabular}
\end{table}

\newpage 

\begin{table}[]
\centering 
\caption{Values for CIFAR-10 classifiers trained with the Cross Entropy loss}
\label{table:cifar-val-xent}
\begin{tabular}{llllll}
\hline
\multicolumn{6}{c}{Lipschitz Values (CIFAR-10) | CrossEntropy Loss}                         \\ \hline
\multicolumn{1}{l|}{Network}   & SeqLip    & SDP       & Fast-Lip  & CLEVER     & ZLip      \\ \hline
\multicolumn{1}{l|}{TinyReLU}  & 1.11\e{3} & 8.51\e{2} & 1.13\e{3} & 6.05\e{-1} & 8.89\e{2} \\
\multicolumn{1}{l|}{SmallReLU} & 1.27\e{4} & 7.23\e{3} & 1.55\e{5} & 1.43\e{0}  & 6.32\e{4} \\
\multicolumn{1}{l|}{MedReLU}   & 4.30\e{5} & 1.78\e{5} & 2.42\e{8} & 2.22\e{0}  & 2.13\e{7} \\
\multicolumn{1}{l|}{TinyTanh}  & 4.12\e{3} & 3.22\e{3} & 5.52\e{3} & 2.28\e{0}  & 4.32\e{3} \\
\multicolumn{1}{l|}{SmallTanh} & 2.61\e{4} & 1.77\e{4} & 1.72\e{5} & 5.73\e{0}  & 7.82\e{4} \\
\multicolumn{1}{l|}{MedTanh}   & 1.28\e{6} & ---       & 1.86\e{8} & 1.26\e{1}  & 2.22\e{7}
\end{tabular}
\end{table}

\begin{table}[]
\centering 
\caption{Values for CIFAR-10 classifiers trained with the PGD loss}
\label{table:cifar-val-pgd}
\begin{tabular}{llllll}
\hline
\multicolumn{6}{c}{Lipschitz Values (CIFAR-10) | PGD Loss}                                  \\ \hline
\multicolumn{1}{l|}{Network}   & SeqLip    & SDP       & Fast-Lip  & CLEVER     & ZLip      \\ \hline
\multicolumn{1}{l|}{TinyReLU}  & 3.01\e{2} & 1.92\e{2} & 4.63\e{1} & 1.17\e{-1} & 3.95\e{1} \\
\multicolumn{1}{l|}{SmallReLU} & 2.03\e{3} & 1.22\e{3} & 2.15\e{4} & 3.62\e{-1} & 7.75\e{3} \\
\multicolumn{1}{l|}{MedReLU}   & 3.61\e{4} & 1.26\e{4} & 1.70\e{7} & 5.22\e{-1} & 1.26\e{6} \\
\multicolumn{1}{l|}{TinyTanh}  & 1.81\e{3} & 1.31\e{3} & 1.05\e{3} & 8.94\e{-1} & 8.91\e{2} \\
\multicolumn{1}{l|}{SmallTanh} & 1.51\e{4} & 7.25\e{3} & 2.62\e{4} & 2.67\e{0}  & 1.35\e{4} \\
\multicolumn{1}{l|}{MedTanh}   & 2.57\e{5} & 9.56\e{4} & 1.41\e{7} & 3.14\e{0}  & 1.81\e{6}
\end{tabular}
\end{table}

\begin{table}[]
\centering 
\caption{Times for CIFAR-10 classifiers trained with the CrossEntropy loss}
\label{table:cifar-time-xent}
\begin{tabular}{llllll}
\hline
\multicolumn{6}{c}{Lipschitz Times (CIFAR-10) | CrossEntropy Loss}                            \\ \hline
\multicolumn{1}{l|}{Network}   & SeqLip     & SDP       & Fast-Lip   & CLEVER    & ZLip       \\ \hline
\multicolumn{1}{l|}{TinyReLU}  & 2.32\e{-1} & 5.02\e{2} & 9.59\e{-2} & 5.91\e{1} & 6.24\e{-1} \\
\multicolumn{1}{l|}{SmallReLU} & 7.14\e{0}  & 2.88\e{3} & 1.07\e{-1} & 1.06\e{2} & 1.16\e{0}  \\
\multicolumn{1}{l|}{MedReLU}   & 7.45\e{1}  & 4.81\e{3} & 6.38\e{-2} & 9.88\e{1} & 1.32\e{0}  \\
\multicolumn{1}{l|}{TinyTanh}  & 6.45\e{-1} & 5.05\e{2} & 1.18\e{-1} & 7.28\e{1} & 8.66\e{-1} \\
\multicolumn{1}{l|}{SmallTanh} & 3.11\e{1}  & 3.25\e{3} & 1.17\e{-1} & 1.22\e{2} & 1.22\e{0}  \\
\multicolumn{1}{l|}{MedTanh}   & 1.09\e{0}  & 6.78\e{3} & 6.34\e{-3} & 3.98\e{0} & 1.89\e{-1}
\end{tabular}
\end{table}

\begin{table}[]
\centering
\caption{Times for CIFAR-10 classifiers trained with the PGD loss}
\label{table:cifar-time-pgd}
\begin{tabular}{llllll}
\hline
\multicolumn{6}{c}{Lipschitz Times (CIFAR-10) | PGD Loss}                                     \\ \hline
\multicolumn{1}{l|}{Network}   & SeqLip     & SDP       & Fast-Lip   & CLEVER    & ZLip       \\ \hline
\multicolumn{1}{l|}{TinyReLU}  & 1.47\e{-1} & 3.94\e{2} & 1.09\e{-1} & 6.74\e{1} & 7.93\e{-1} \\
\multicolumn{1}{l|}{SmallReLU} & 4.11\e{1}  & 2.52\e{3} & 9.89\e{-2} & 9.12\e{1} & 1.07\e{0}  \\
\multicolumn{1}{l|}{MedReLU}   & 5.22\e{1}  & 4.01\e{3} & 1.14\e{-1} & 1.89\e{2} & 2.39\e{0}  \\
\multicolumn{1}{l|}{TinyTanh}  & 8.12\e{-2} & 4.60\e{2} & 1.13\e{-1} & 7.02\e{1} & 8.25\e{-1} \\
\multicolumn{1}{l|}{SmallTanh} & 4.97\e{0}  & 3.28\e{3} & 1.13\e{-1} & 1.22\e{2} & 1.23\e{0}  \\
\multicolumn{1}{l|}{MedTanh}   & 1.13\e{0}  & 3.90\e{3} & 5.77\e{-3} & 3.98\e{0} & 1.89\e{-1}
\end{tabular}
\end{table}

\end{document}